\documentclass{article}%
\usepackage{iclr2017_conference,times}
\usepackage{amsmath}
\usepackage{amsthm}
\usepackage{amssymb}
\usepackage{titlesec}
\usepackage{xcolor}
\usepackage{hyperref}
\usepackage[normalem]{ulem}
\usepackage[pdftex]{graphicx}
\usepackage[tight]{subfigure}
\usepackage{caption}
\usepackage{amsfonts}%
\providecommand{\U}[1]{\protect\rule{.1in}{.1in}}
\iclrfinalcopy
\newtheoremstyle{thm}{1.0em}{1.0em}{\itshape}{}{\bfseries}{.}{.5em}{}
\theoremstyle{thm}

\newtheorem{proposition}{Proposition}
\newtheoremstyle{rem}{1.0em}{1.0em}{\normalfont}{}{\bfseries}{.}{.5em}{}
\theoremstyle{rem}

\newenvironment{proof}[1][Proof]{\par\vspace{1.0em}\noindent\textbf{#1. }}{\hfill $\Box$\par\vspace{1.0em}}

\ifx\pdfoutput\relax\let\pdfoutput=\undefined\fi
\newcount\msipdfoutput
\ifx\pdfoutput\undefined\else
\ifcase\pdfoutput\else
\ifx\paperwidth\undefined\else
\ifdim\paperheight=0pt\relax\else\pdfpageheight\paperheight\fi
\ifdim\paperwidth=0pt\relax\else\pdfpagewidth\paperwidth\fi
\fi\fi\fi
\begin{document}

\title{An Information-Theoretic Framework for Fast and Robust Unsupervised Learning
via Neural Population Infomax}
\author{Wentao Huang \& Kechen Zhang \\
Department of Biomedical Engineering\\
Johns Hopkins University School of Medicine\\
Baltimore, MD 21205, USA \\
\texttt{\{whuang21,kzhang4\}@jhmi.edu}\\
}
\maketitle

\begin{abstract}
A framework is presented for unsupervised learning of representations based on
infomax principle for large-scale neural populations. We use an asymptotic
approximation to the Shannon's mutual information for a large neural
population to demonstrate that a good initial approximation to the global
information-theoretic optimum can be obtained by a hierarchical infomax
method. Starting from the initial solution, an efficient algorithm based on
gradient descent of the final objective function is proposed to learn
representations from the input datasets, and the method works for complete,
overcomplete, and undercomplete bases. As confirmed by numerical experiments,
our method is robust and highly efficient for extracting salient features from
input datasets. Compared with the main existing methods, our algorithm has a
distinct advantage in both the training speed and the robustness of
unsupervised representation learning. Furthermore, the proposed method is easily extended to the supervised or unsupervised model for training deep structure networks.

\end{abstract}

\section{Introduction}

\label{Introduction}

How to discover the unknown structures in data is a key task for machine
learning. Learning good representations from observed data is important
because a clearer description may help reveal the underlying structures.
Representation learning has drawn considerable attention in recent years
\citep{Bengio(2013-representation)}. One category of algorithms for
unsupervised learning of representations is based on probabilistic models
\citep{Lewicki(2000-learning),Hinton(2006-reducing),Lee(2008-IP-sparse)}, such
as maximum likelihood (ML) estimation, maximum a posteriori (MAP) probability
estimation, and related methods. Another category of algorithms is based on
reconstruction error or generative criterion
\citep{Olshausen(1996-emergence),Aharon(2006-k),Vincent(2010-stacked),Mairal(2010-online),Goodfellow(2014-IP-generative)},
and the objective functions usually involve squared errors with additional
constraints. Sometimes the reconstruction error or generative criterion may
also have a probabilistic interpretation \citep{Olshausen(1997-sparse),Vincent(2010-stacked)}.

Shannon's information theory is a powerful tool for description of stochastic
systems and could be utilized to provide a characterization for good
representations \citep{Vincent(2010-stacked)}. However, computational
difficulties associated with Shannon's mutual information (MI)
\citep{Shannon(1948-mathematical)} have hindered its wider applications. The
Monte Carlo (MC) sampling \citep{Yarrow(2012-fisher)} is a convergent method
for estimating MI with arbitrary accuracy, but its computational inefficiency
makes it unsuitable for difficult optimization problems especially in the
cases of high-dimensional input stimuli and large population networks. Bell
and Sejnowski \citep{Bell(1995-information),Bell(1997-independent)} have
directly applied the infomax approach \citep{Linsker(1988-self)} to
independent component analysis (ICA) of data with independent non-Gaussian
components assuming additive noise, but their method requires that the number
of outputs be equal to the number of inputs. The extensions of ICA to
overcomplete or undercomplete bases incur increased algorithm complexity and
difficulty in learning of parameters \citep{Lewicki(2000-learning),Kreutz-Delgado(2003-dictionary),Karklin(2011-IP-efficient)}.

Since Shannon MI is closely related to ML and MAP
\citep{Huang(2016-information)}, the algorithms of representation learning
based on probabilistic models should be amenable to information-theoretic
treatment. Representation learning based on reconstruction error could be
accommodated also by information theory, because the inverse of Fisher
information (FI) is the Cram\'{e}r-Rao lower bound on the mean square decoding
error of any unbiased decoder \citep{Rao(1945-information)}. Hence minimizing
the reconstruction error potentially maximizes a lower bound on the MI \citep{Vincent(2010-stacked)}.

Related problems arise also in neuroscience. It has long been suggested that
the real nervous systems might approach an information-theoretic optimum for
neural coding and computation
\citep{Barlow(1961-possible),Atick(1992-could),Borst(1999-information)}.
However, in the cerebral cortex, the number of neurons is huge, with about
$10^{5}$ neurons under a square millimeter of cortical surface
\citep{Carlo(2013-structural)}. It has often been computationally intractable
to precisely characterize information coding and processing in large neural populations.

To address all these issues, we present a framework for unsupervised learning
of representations in a large-scale nonlinear feedforward model based on
infomax principle with realistic biological constraints such as neuron models
with Poisson spikes. First we adopt an objective function based on an
asymptotic formula in the large population limit for the MI between the
stimuli and the neural population responses \citep{Huang(2016-information)}.
Since the objective function is usually nonconvex, choosing a good initial
value is very important for its optimization. Starting from an initial value,
we use a hierarchical infomax approach to quickly find a tentative global
optimal solution for each layer by analytic methods. Finally, a fast
convergence learning rule is used for optimizing the final objective function
based on the tentative optimal solution. Our algorithm is robust and can learn
complete, overcomplete or undercomplete basis vectors quickly from different
datasets. Experimental results showed that the convergence rate of our method
was significantly faster than other existing methods, often by an order of
magnitude. More importantly, the number of output units processed by our
method can be very large, much larger than the number of inputs. As far as we
know, no existing model can easily deal with this situation.

\section{Methods}

\subsection{Approximation of Mutual Information for Neural Populations}

\label{Approximation}

Suppose the input $\mathbf{x}$ is a $K$-dimensional vector, $\mathbf{x}%
=(x_{1},\cdots,x_{K})^{T}$, the outputs of $N$ neurons are denoted by a
vector, $\mathbf{r}=(r_{1},\cdots,r_{N})^{T}$, where we assume $N$ is large,
generally $N\gg K$. We denote random variables by upper case letters, e.g.,
random variables $X$ and $R$, in contrast to their vector values $\mathbf{x}$
and $\mathbf{r}$. The MI between $X$ and $R$ is defined by
$I(X;R)=\left\langle \ln\frac{p(\mathbf{x}|\mathbf{r})}{p(\mathbf{x}%
)}\right\rangle _{\mathbf{r},\mathbf{x}}$, where $\left\langle \cdot
\right\rangle _{\mathbf{r},\mathbf{x}}$ denotes the expectation with respect
to the probability density function (PDF) $p(\mathbf{r},\mathbf{x})$.

Our goal is to maxmize MI $I(X;R)$ by finding the optimal PDF $p(\mathbf{r}%
|\mathbf{x})$ under some constraint conditions, assuming that $p(\mathbf{r}%
|\mathbf{x})$ is characterized by a noise model and activation functions
$f(\mathbf{x};{\boldsymbol{\theta}}_{n})$ with parameters ${\boldsymbol{\theta
}}_{n}$ for the $n$-th neuron ($n=1,\cdots,N$). In other words, we optimize
$p(\mathbf{r}|\mathbf{x})$ by solving for the optimal parameters
${\boldsymbol{\theta}}_{n}$. Unfortunately, it is intractable in most cases to
solve for the optimal parameters that maximizes $I(X;R)$. However, if
$p(\mathbf{x})$ and $p(\mathbf{r}|\mathbf{x})$ are twice continuously
differentiable for almost every $\mathbf{x}\in%
\mathbb{R}
^{K}$, then for large $N$ we can use an asymptotic formula to approximate the
true value of $I(X;R)$ with high accuracy \citep{Huang(2016-information)}:
\begin{equation}
I(X;R)\simeq I_{G}=\frac{1}{2}\left\langle \ln\left(  \det\left(
\frac{\mathbf{G}(\mathbf{x})}{2\pi e}\right)  \right)  \right\rangle
_{\mathbf{x}}+H(X)\text{,} \label{Ia}%
\end{equation}
where $\det\left(  \cdot\right)  $ denotes the matrix determinant and
$H(X)=-\left\langle \ln p(\mathbf{x})\right\rangle _{\mathbf{x}}$ is the
stimulus entropy,
\begin{align}
&  \mathbf{G}(\mathbf{x})=\mathbf{J}(\mathbf{x})+\mathbf{P}\left(
\mathbf{x}\right)  \text{,}\label{Gx}\\
&  \mathbf{J}(\mathbf{x})=-\left\langle \frac{\partial^{2}\ln p\left(
\mathbf{r}|\mathbf{x}\right)  }{\partial\mathbf{x}\partial\mathbf{x}^{T}%
}\right\rangle _{\mathbf{r}|\mathbf{x}}\text{,}\label{Jx}\\
&  \mathbf{P}(\mathbf{x})=-\frac{\partial^{2}\ln p\left(  \mathbf{x}\right)
}{\partial\mathbf{x}\partial\mathbf{x}^{T}}\text{.} \label{Px}%
\end{align}
Assuming independent noises in neuronal responses, we have $p(\mathbf{r}%
|\mathbf{x})=\prod_{n=1}^{N}p(r_{n}|\mathbf{x};\boldsymbol{\theta}_{n})$, and
the Fisher information matrix becomes $\mathbf{J}(\mathbf{x})\approx
N\sum_{k=1}^{K_{1}}\alpha_{k}\mathbf{S}(\mathbf{x};{\boldsymbol{\theta}}_{k}%
)$, where $\mathbf{S}(\mathbf{x};\boldsymbol{\theta}_{k})=\left\langle
\frac{\partial\ln p(r|\mathbf{x};\boldsymbol{\theta}_{k})}{\partial\mathbf{x}%
}\frac{\partial\ln p(r|\mathbf{x};\boldsymbol{\theta}_{k})}{\partial
\mathbf{x}^{T}}\right\rangle _{r|\mathbf{x}}$ and $\alpha_{k}>0$
($k=1,\cdots,K_{1}$) is the population{ }density of parameter
${\boldsymbol{\theta}}_{k}$, with $\sum_{k=1}^{K_{1}}\alpha_{k}=1$, and $1\leq
K_{1}\leq N$ (see Appendix \ref{formulas} for details). Since the cerebral
cortex usually forms functional column structures and each column is composed
of neurons with the same properties \citep{Hubel(1962-receptive)}, the
positive integer $K_{1}$\ can be regarded as the number of distinct classes in
the neural population.

Therefore, given the activation function $f(\mathbf{x};{\boldsymbol{\theta}%
}_{k})$, our goal becomes to find the optimal population distribution density
$\alpha_{k}$ of parameter vector ${\boldsymbol{\theta}}_{k}$ so that the MI
between the stimulus $\mathbf{x}$ and the response $\mathbf{r}$ is maximized.
By Eq. (\ref{Ia}), our optimization problem can be stated as follows:
\begin{align}
&  \mathsf{minimize}\mathrm{\;}Q_{G}[\left\{  \alpha_{k}\right\}  ]=-\frac
{1}{2}\left\langle \ln\left(  \det\left(  \mathbf{G}(\mathbf{x})\right)
\right)  \right\rangle _{\mathbf{x}}\text{,}\label{minQa}\\
&  \mathsf{subject\;to\;}\sum_{k=1}^{K_{1}}\alpha_{k}=1\text{, }\alpha
_{k}>0\text{, }\forall k=1,\cdots,K_{1}\text{.} \label{Consa}%
\end{align}

Since $Q_{G}[\left\{  \alpha_{k}\right\}  ]$ is a convex function of $\left\{
\alpha_{k}\right\}  $ \citep{Huang(2016-information)}, we can readily find the
optimal solution for small $K$ by efficient numerical methods. For large $K$,
however, finding an optimal solution by numerical methods becomes intractable.
In the following we will propose an alternative approach to this problem.
Instead of directly solving for the density distribution $\left\{  \alpha
_{k}\right\}  $, we optimize the parameters $\left\{  \alpha_{k}\right\}  $
and $\left\{  \boldsymbol{\theta}_{k}\right\}  $ simultaneously under a
hierarchical infomax framework.

\subsection{Hierarchical Infomax}

\label{HieInfomax}

For clarity, we consider neuron model with Poisson spikes although our method
is easily applicable to other noise models. The activation function
$f(\mathbf{x};{\boldsymbol{\theta}}_{n})$ is generally a nonlinear function,
such as sigmoid and rectified linear unit (ReLU)
\citep{Nair(2010-IP-rectified)}. We assume that the nonlinear function for the
\textit{n}-th neuron has the following form: $f(\mathbf{x};{\boldsymbol{\theta
}}_{n})=\tilde{f}(y_{n};{\boldsymbol{\tilde{\theta}}}_{n})$, where
\begin{equation}
y_{n}=\mathbf{w}_{n}^{T}\mathbf{x}\text{.} \label{y=wx}%
\end{equation}
with $\mathbf{w}_{n}$ being a $K$-dimensional weights vector, $\tilde{f}%
(y_{n};{\boldsymbol{\tilde{\theta}}}_{n})$ is a nonlinear function,
${\boldsymbol{\theta}}_{n}=(\mathbf{w}_{n}^{T},{\boldsymbol{\tilde{\theta}}%
}_{n}^{T})^{T}$ and ${\boldsymbol{\tilde{\theta}}}_{n}$ are the parameter
vectors ($n=1,\cdots,N$).

In general, it is very difficult to find the optimal parameters,
${\boldsymbol{\theta}}_{n}$, $n=1,\cdots,N$, for the following reasons. First,
the number of output neurons $N$ is very large, usually $N\gg K$. Second, the
activation function $f(\mathbf{x};{\boldsymbol{\theta}}_{n})$ is a nonlinear
function, which usually leads to a nonconvex optimization problem. For
nonconvex optimization problems, the selection of initial values often has a
great influence on the final optimization results. Our approach meets these
challenges by making better use of the large number of neurons and by finding
good initial values by a hierarchical infomax method.

We divide the nonlinear transformation into two stages, mapping first from
$\mathbf{x}$ to $y_{n}$ ($n=1,\cdots,N$), and then from $y_{n}$ to $\tilde
{f}(y_{n};{\boldsymbol{\tilde{\theta}}}_{n})$, where $y_{n}$\ can be regarded
as the membrane potential of the \textit{n}-th neuron, and $\tilde{f}%
(y_{n};{\boldsymbol{\tilde{\theta}}}_{n})$ as its firing rate. As with the
real neurons, we assume that the membrane potential is corrupted by noise:
\begin{equation}
\breve{Y}_{n}=Y_{n}+Z_{n}\text{,} \label{Yn}%
\end{equation}
where $Z_{n}\sim{\mathcal{N}\left(  0\text{,\thinspace}\sigma^{2}\right)  }$
is a normal distribution with mean $0$ and variance $\sigma^{2}$. Then the
mean membrane potential of the \textit{k}-th class subpopulation with
$N_{k}=N\alpha_{k}$\ neurons is given by
\begin{align}
\bar{Y}_{k}  &  =\frac{1}{N_{k}}\sum_{n=1}^{N_{k}}\breve{Y}_{k_{n}}=Y_{k}%
+\bar{Z}_{k}\text{\textbf{, }}k=1,\cdots,K_{1}\text{,}\label{vyk}\\
\bar{Z}_{k}  &  \sim{\mathcal{N(}0,\,N_{k}^{-1}\sigma^{2})}\text{.} \label{Z-}%
\end{align}
Define vectors $\mathbf{\breve{y}}=(\breve{y}_{1},\cdots,\breve{y}_{N})^{T}$,
$\mathbf{\bar{y}}=(\bar{y}_{1},\cdots,\bar{y}_{K_{1}})^{T}$ and $\mathbf{y}%
=(y_{1},\cdots,y_{K_{1}})^{T}$, where $y_{k}=\mathbf{w}_{k}^{T}\mathbf{x}$
($k=1,\cdots,K_{1}$). Notice that $\breve{y}_{n}$ ($n=1,\cdots,N$)\ is also
divided into $K_{1}$\ classes, the same as for $r_{n}$. If we assume
$f(\mathbf{x};{\boldsymbol{\theta}}_{k})=\tilde{f}(\bar{y}_{k}%
;{\boldsymbol{\tilde{\theta}}}_{k})$, i.e. assuming an additive Gaussian noise
for $y_{n}$ (see Eq. \ref{vyk}), then the random variables $X$, $Y$,
$\breve{Y}$, $\bar{Y}$ and $R$ form a Markov chain, denoted by $X\rightarrow
Y\rightarrow\breve{Y}\rightarrow\bar{Y}\rightarrow R$ (see Figure~\ref{Fig0}),
and we have the following proposition (see Appendix
\ref{Proof of Proposition 1}).

\begin{proposition}
\label{Proposition 1} With the random variables $X$, $Y$, $\breve{Y}$,
$\bar{Y}$, $R$ and Markov chain $X\rightarrow Y\rightarrow\breve{Y}%
\rightarrow\bar{Y}\rightarrow R$, the following equations hold,%
\begin{align}
I(X;R)  &  =I(Y;R)\leq I(\breve{Y};R)\leq I(\bar{Y};R)\text{,}%
\label{proposition_1}\\
I(X;R)  &  \leq I(X;\bar{Y})=I(X;\breve{Y})\leq I(X;Y)\text{,}
\label{proposition_2}%
\end{align}
and for large $N_{k}$ ($k=1,\cdots,K_{1}$),%
\begin{align}
I(\breve{Y};R)  &  \simeq I(\bar{Y};R)\simeq I(Y;R)=I(X;R)\text{,}%
\label{proposition_3a}\\
I(X;Y)  &  \simeq I(X;\bar{Y})=I(X;\breve{Y})\text{.} \label{proposition_3b}%
\end{align}

\end{proposition}

\subfiglabelskip=0pt \begin{figure}[t]
\vskip -0.2in \centering
\includegraphics[width= .8\linewidth]{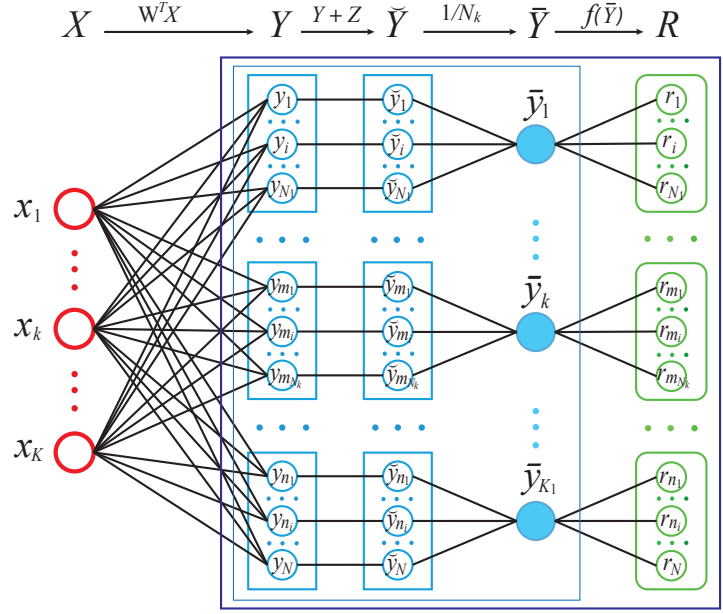} \hspace{0pt}
\vskip -0.15in\caption{A neural network interpretaton for random variables
$X$, $Y$, $\breve{Y}$,$\bar{Y}$, $R$.}%
\label{Fig0}%
\end{figure}

A major advantage of incorporating membrane noise is that it facilitates
finding the optimal solution by using the infomax principle. Moreover, the
optimal solution obtained this way is more robust; that is, it discourages
overfitting and has a strong ability to resist distortion. With vanishing
noise $\sigma^{2}\rightarrow0$, we have $\bar{Y}_{k}\rightarrow Y_{k}$,
$\tilde{f}(\bar{y}_{k};{\boldsymbol{\tilde{\theta}}}_{k})\simeq\tilde{f}%
(y_{k};{\boldsymbol{\tilde{\theta}}}_{k})=f(\mathbf{x};{\boldsymbol{\theta}%
}_{k})$, so that Eqs.~(\ref{proposition_3a}) and (\ref{proposition_3b}) hold
as in the case of large $N_{k}$.

To optimize MI $I(Y;R)$, the probability distribution of random variable
$Y$,\ $p(\mathbf{y})$, needs to be determined, i.e. maximizing $I(Y;R)$ about
$p(\mathbf{y})$\ under some constraints should yield an optimal distribution:
$p^{\ast}(\mathbf{y})=\arg\max_{p(\mathbf{y})}I(Y;R)$. Let $\mathcal{C}%
=\max_{p(\mathbf{y})}I\left(  Y;R\right)  $ be the channel capacity of neural
population coding, and we always have $I(X;R)\leq\mathcal{C}$
\citep{Huang(2016-information)}. To find a suitable linear transformation from
$X$ to $Y$ that is compatible with this distribution $p^{\ast}(\mathbf{y})$, a
reasonable choice is to maximize $I(X;\breve{Y})$ $\left(  \leq I(X;Y)\right)
$, where $\breve{Y}$\ is a noise-corrupted version of $Y$. This implies
minimum information loss in the first transformation step. However, there may
exist many transformations from $X$ to $\breve{Y}$ that maximize
$I(X;\breve{Y})$ (see Appendix \ref{1st}). Ideally, if we can find a
transformation that maximizes both $I(X;\breve{Y})$ and $I(Y;R)$
simultaneously, then $I(X;R)$ reaches its maximum value: $I(X;R)=\max
_{p(\mathbf{y})}I\left(  Y;R\right)  =\mathcal{C}$.

From the discussion above we see that maximizing $I(X;R)$ can be divided into
two steps, namely, maximizing $I(X;\breve{Y})$ and maximizing $I(Y;R)$. The
optimal solutions of $\max I(X;\breve{Y})$ and $\max I(Y;R)$ will provide a
good initial approximation that tend to be very close to the optimal solution
of $\max I(X;R)$.

Similarly, we can extend this method to multilayer neural population{
}networks. For example, a two-layer network with outputs $R^{(1)}$ and
$R^{(2)}${ }form a Markov chain, $X\rightarrow\tilde{R}^{(1)}\rightarrow
R^{(1)}\rightarrow\bar{R}^{(1)}\rightarrow R^{(2)}$, where random variable
$\tilde{R}^{(1)}$ is similar to $Y$, random variable $R^{(1)}$ is similar to
$\breve{Y}$, and $\bar{R}^{(1)}$ is similar to $\bar{Y}$ in the above. Then we
can show that the optimal solution of $\max I(X;R^{(2)})$ can be approximated
by the solutions of $\max I(X;R^{(1)})$ and $\max I(\tilde{R}^{(1)};R^{(2)})$,
with $I(\tilde{R}^{(1)};R^{(2)})\simeq I(\bar{R}^{(1)};R^{(2)})$.

More generally, consider a highly nonlinear feedforward neural network that
maps the input $\mathbf{x}$ to output $\mathbf{z}$, with $\mathbf{z}%
=F(\mathbf{x;}{\boldsymbol{\theta}})=h_{L}\circ\cdots\circ h_{1}\left(
\mathbf{x}\right)  $, where $h_{l}$ ($l=1,\cdots,L$) is a linear or nonlinear
function \citep{Montufar(2014-IP-number)}. We aim to find the optimal
parameter ${\boldsymbol{\theta}}$ by maximizing $I\left(  X;Z\right)  $. It is
usually difficult to solve the optimization problem when there are many local
extrema for $F(\mathbf{x;}{\boldsymbol{\theta}})$. However, if each function
$h_{l}$\ is easy to optimize, then we can use the hierarchical infomax method
described above to get a good initial approximation to its global optimization
solution, and go from there to find the final optimal solution. This
information-theoretic consideration from the neural population coding point of
view may help explain why deep structure networks with unsupervised
pre-training have a powerful ability for learning representations.

\subsection{The Objective Function}

\label{ObjFun}

The optimization processes for maximizing $I(X;\breve{Y})$ and maximizing
$I(Y;R)$\ are discussed in detail in Appendix \ref{A_HieOpt}. First, by
maximizing $I(X;\breve{Y})$ (see Appendix \ref{1st} for details), we can get
the optimal weight parameter $\mathbf{w}_{k}$ ($k=1,\cdots,K_{1}$, see Eq.
\ref{y=wx})\ and its population{ }density $\alpha_{k}$ (see Eq. \ref{Consa})
which satisfy%
\begin{align}
&  \mathbf{W}={\left[  \mathbf{w}_{1},\cdots,\mathbf{w}_{K_{1}}\right]  }%
={a}\mathbf{U}_{0}{\boldsymbol{\Sigma}_{0}^{-1/2}\mathbf{C}}\text{,}%
\label{W_0}\\
&  \alpha_{1}=\cdots=\alpha_{K_{1}}=K_{1}^{-1}\text{,} \label{A_0}%
\end{align}
where $a=\sqrt{{K_{1}}K_{0}^{-1}}$, ${\mathbf{C}}=[\mathbf{c}_{1}%
,\cdots,\mathbf{c}_{K_{1}}]\in%
\mathbb{R}
^{K_{0}\times K_{1}}$, ${\mathbf{CC}}^{T}=\mathbf{I}_{K_{0}}$, $\mathbf{I}%
_{K_{0}}$ is a $K_{0}\times K_{0}$ identity matrix with integer $K_{0}%
\in\left[  1,K\right]  $, the diagonal matrix $\mathbf{\Sigma}_{0}\in%
\mathbb{R}
^{K_{0}\times K_{0}}$ and matrix $\mathbf{U}_{0}\in%
\mathbb{R}
^{K\times K_{0}}$ are given in (\ref{sigma0}) and (\ref{U0}), with $K_{0}$
given by Eq. (\ref{K0}). Matrices $\mathbf{\Sigma}_{0}$ and $\mathbf{U}_{0}%
$\ can be obtained by $\mathbf{\Sigma}$ and $\mathbf{U}$ with $\mathbf{U}%
_{0}^{T}\mathbf{U}_{0}=\mathbf{I}_{K_{0}}$ and $\mathbf{U}_{0}%
\mathbf{\mathbf{\Sigma}}_{0}\mathbf{U}_{0}^{T}\approx\mathbf{U\mathbf{\Sigma
}U}^{T}\approx\left\langle \mathbf{xx}^{T}\right\rangle _{\mathbf{x}}$ (see
Eq. \ref{Sig}). The optimal weight parameter $\mathbf{w}_{k}$ (\ref{W_0})
means that the input variable $\mathbf{x}$ must first undergo a whitening-like
transformation $\mathbf{\hat{x}}=\mathbf{\Sigma}_{0}^{-1/2}\mathbf{U}_{0}%
^{T}\mathbf{x}$, and then goes through the transformation $\mathbf{y}%
=a{\mathbf{C}}^{T}\mathbf{\hat{x}}$, with matrix ${\mathbf{C}}$ to be
optimized below. Note that weight matrix $\mathbf{W}$ satisfies
$\mathrm{rank}(\mathbf{W})=\min(K_{0},K_{1})$, which is a low rank matrix, and its low dimensionality helps reduce overfitting during training (see Appendix \ref{1st}).

By maximizing $I\left(  Y;R\right)  $ (see Appendix \ref{2nd}), we further
solve the the optimal parameters ${\boldsymbol{\tilde{\theta}}}_{k}$ for
the nonlinear functions $\tilde{f}(y_{k};{\boldsymbol{\tilde{\theta}}}_{k})$,
$k=1,\cdots,K_{1}$. Finally, the objective function for our optimization
problem (Eqs.~\ref{minQa} and \ref{Consa}) turns into (see Appendix
\ref{final} for details):%
\begin{align}
&  \mathsf{minimize}\;\;{Q\left[  {\mathbf{C}}\right]  =}-\frac{1}%
{2}\left\langle \ln\left(  \det\left(  {\mathbf{C\boldsymbol{\hat{\Phi}}C}%
}^{T}\right)  \right)  \right\rangle _{\mathbf{\hat{x}}}\text{,}\label{QC_0}\\
&  \mathsf{subject\;to}\;{\mathbf{CC}}^{T}=\mathbf{I}_{K_{0}}\text{,}%
\label{CC_0}%
\end{align}
where $\boldsymbol{\hat{\Phi}}={{\mathrm{diag}}}\left(  \phi(\hat{y}_{1}%
)^{2},\cdots,\phi(\hat{y}_{K_{1}})^{2}\right)  $, $\phi(\hat{y}_{k}%
)=a^{-1}\left\vert \partial g_{k}(\hat{y}_{k})/\partial\hat{y}_{k}\right\vert
$ ($k=1,\cdots,K_{1}$), $g_{k}(\hat{y}_{k})=2\sqrt{\tilde{f}(\hat{y}%
_{k};{\boldsymbol{\tilde{\theta}}}_{k})}$, $\hat{y}_{k}=a^{-1}y_{k}%
=\mathbf{c}_{k}^{T}\mathbf{\hat{x}}$, and $\mathbf{\hat{x}}=\mathbf{\Sigma
}_{0}^{-1/2}\mathbf{U}_{0}^{T}\mathbf{x}$. We apply the gradient descent
method to optimize the objective function, with the gradient of $Q{\left[
{\mathbf{C}}\right]  }$ given by:%
\begin{equation}
\frac{dQ{\left[  {\mathbf{C}}\right]  }}{d{\mathbf{C}}}=-\left\langle \left(
{\mathbf{C\boldsymbol{\hat{\Phi}}C}}^{T}\right)  ^{-1}\mathbf{C}%
\boldsymbol{\hat{\Phi}}+\mathbf{\hat{x}}\boldsymbol{\omega}^{T}\right\rangle
_{\mathbf{\hat{x}}}\text{,}\label{dQC_0}%
\end{equation}
where $\boldsymbol{\omega}=\left(  \omega_{1}{,}\cdots,\omega{_{K_{1}}%
}\right)  ^{T}$, $\omega_{k}=\phi(\hat{y}_{k})\phi^{\prime}(\hat{y}%
_{k})\mathbf{c}_{k}^{T}\left(  {\mathbf{C\boldsymbol{\hat{\Phi}}C}}%
^{T}\right)  ^{-1}\mathbf{c}_{k}$, $k=1,\cdots,K_{1}$.

When ${K_{0}}=K_{1}$ (or ${K_{0}}>K_{1}$), the objective function $Q{\left[
{\mathbf{C}}\right]  }$ can be reduced to a simpler form, and its gradient is
also easy to compute (see Appendix \ref{Alg.1}). However, when ${K_{0}}<K_{1}%
$, it is computationally expensive to update $\mathbf{C}$ by applying the
gradient of $Q{\left[  {\mathbf{C}}\right]  }$ directly, since it requires
matrix inversion for every $\mathbf{\hat{x}}$. We use another objective
function $\hat{Q}{\left[  {\mathbf{C}}\right]  }$ (see Eq. \ref{obj2}) which
is an approximation to $Q{\left[  {\mathbf{C}}\right]  }$, but its gradient is
easier to compute (see Appendix \ref{Alg.2}). The function $\hat{Q}{\left[
{\mathbf{C}}\right]  }$ is the approximation of $Q{\left[  {\mathbf{C}%
}\right]  }$, ideally they have the same optimal solution for the parameter
$\mathbf{C}$.

Usually, for optimizing the objective in Eq.~\ref{QC_0}, the orthogonality
constraint (Eq. \ref{CC_0}) is unnecessary. However, this orthogonality
constraint can accelerate the convergence rate if we employ it for the initial
iteration to update $\mathbf{C}$ (see Appendix \ref{SupExp}).

\section{Experimental Results}

\label{results}

We have applied our methods to the natural images from Olshausen's image
dataset \citep{Olshausen(1996-emergence)} and the images of handwritten digits
from MNIST dataset \citep{LeCun(1998-gradient)} using Matlab 2016a on a
computer with 12 Intel CPU cores (2.4 GHz). The gray level of each raw image
was normalized to the range of $0$ to $1$. $M\ $image patches with size
$w\times w=K$ for training were randomly sampled from the images. We used the
Poisson neuron model with a modified sigmoidal tuning function $\tilde
{f}(y;{\boldsymbol{\tilde{\theta}}})=\frac{1}{4\left(  1+\exp\left(  -\beta
y-b\right)  \right)  ^{2}}$, with $g(y)=2\sqrt{\tilde{f}(y;{\boldsymbol{\tilde
{\theta}}})}=\frac{1}{1+\exp\left(  -\beta y-b\right)  }$, where
${\boldsymbol{\tilde{\theta}}}=\left(  \beta,b\right)  ^{T}$. We obtained the
initial values (see Appendix \ref{2nd}): $b_{0}=0$ and $\beta_{0}%
\approx1.81\sqrt{{K_{1}}K_{0}^{-1}}$. For our experiments, we set
$\beta=0.5\beta_{0}$ for iteration epoch $t=1,\cdots,t_{0}$ and $\beta
=\beta_{0}$ for $t=t_{0}+1,\cdots,t_{\max}$, where $t_{0}=50$.

Firstly, we tested the case of $K=K_{0}=K_{1}=144$ and randomly sampled
$M=10^{5}$ image patches with size $12\times12$ from the Olshausen's natural
images, assuming that $N=10^{6}$ neurons were divided into $K_{1}%
=144$\ classes and $\epsilon=1$ (see Eq. \ref{K0} in Appendix). The input
patches were preprocessed by the ZCA whitening filters (see Eq. \ref{xv}). To
test our algorithms, we chose the batch size to be equal to the number of
training samples $M$, although we could also choose a smaller batch size. We
updated the matrix $\mathbf{C}$ from a random start, and set parameters
$t_{\max}=300$, ${v}_{1}=0.4$, and $\tau=0.8$ for all experiments.

In this case, the optimal solution $\mathbf{C}$ looked similar to the optimal
solution of IICA \citep{Bell(1997-independent)}. We also compared with the
fast ICA algorithm (FICA) \citep{Hyvaerinen(1999-fast)}, which is faster than
IICA. We also tested the restricted Boltzmann machine (RBM)
\citep{Hinton(2006-fast)} for a unsupervised learning of representations, and
found that it could not easily learn Gabor-like filters from Olshausen's image
dataset as trained by contrastive divergence. However, an improved method by
adding a sparsity constraint on the output units, e.g., sparse RBM (SRBM)
\citep{Lee(2008-IP-sparse)} or sparse autoencoder
\citep{Hinton(2010-practical)}, could attain Gabor-like filters from this
dataset. Similar results with Gabor-like filters were also reproduced by the
denoising autoencoders \citep{Vincent(2010-stacked)}, which method requires a
careful choice of parameters, such as noise level, learning rate, and batch size.

In order to compare our methods, i.e. Algorithm 1 (Alg.1, see Appendix
\ref{Alg.1}) and Algorithm 2 (Alg.2, see Appendix \ref{Alg.2}), with other
methods, i.e. IICA, FICA and SRBM, we implemented these algorithms using the
same initial weights and the same training data set (i.e. $10^{5}$ image
patches preprocessed by the ZCA whitening filters). To get a good result by
IICA, we must carefully select the parameters; we set the batch size as $50$,
the initial learning rate as $0.01$, and final learning rate as $0.0001$, with
an exponential decay with the epoch of iterations. IICA tends to have a faster
convergence rate for a bigger batch size but it may become harder to escape
local minima. For FICA, we chose the nonlinearity function $f(u)=\log\cosh(u)$
as contrast function, and for SRBM, we set the sparseness control constant $p$
as $0.01$ and $0.03$. The number of epoches for iterations was set to $300$
for all algorithms. Figure~\ref{Fig1} shows the filters learned by our methods
and other methods. Each filter in Figure~\ref{Fig1a} corresponds to a column
vector of matrix ${\mathbf{\check{C}}}$ (see Eq. \ref{Cv}), where each vector
for display is normalized by $\mathbf{\check{c}}_{k}\leftarrow\mathbf{\check
{c}}_{k}/\max(|\check{c}_{1,k}|,\cdots,|\check{c}_{K,k}|)$, $k=1,\cdots,K_{1}%
$. The results in Figures~\ref{Fig1a}, \ref{Fig1b} and \ref{Fig1c} look very
similar to one another, and slightly different from the results in
Figure~\ref{Fig1d} and \ref{Fig1e}. There are no Gabor-like filters in
Figure~\ref{Fig1f}, which corresponds to SRBM with $p=0.03$.

\subfiglabelskip=0pt \begin{figure}[th]
\vskip -0.15in \centering
\subfigure[]{\label{Fig1a}
\includegraphics[width= .315\linewidth]{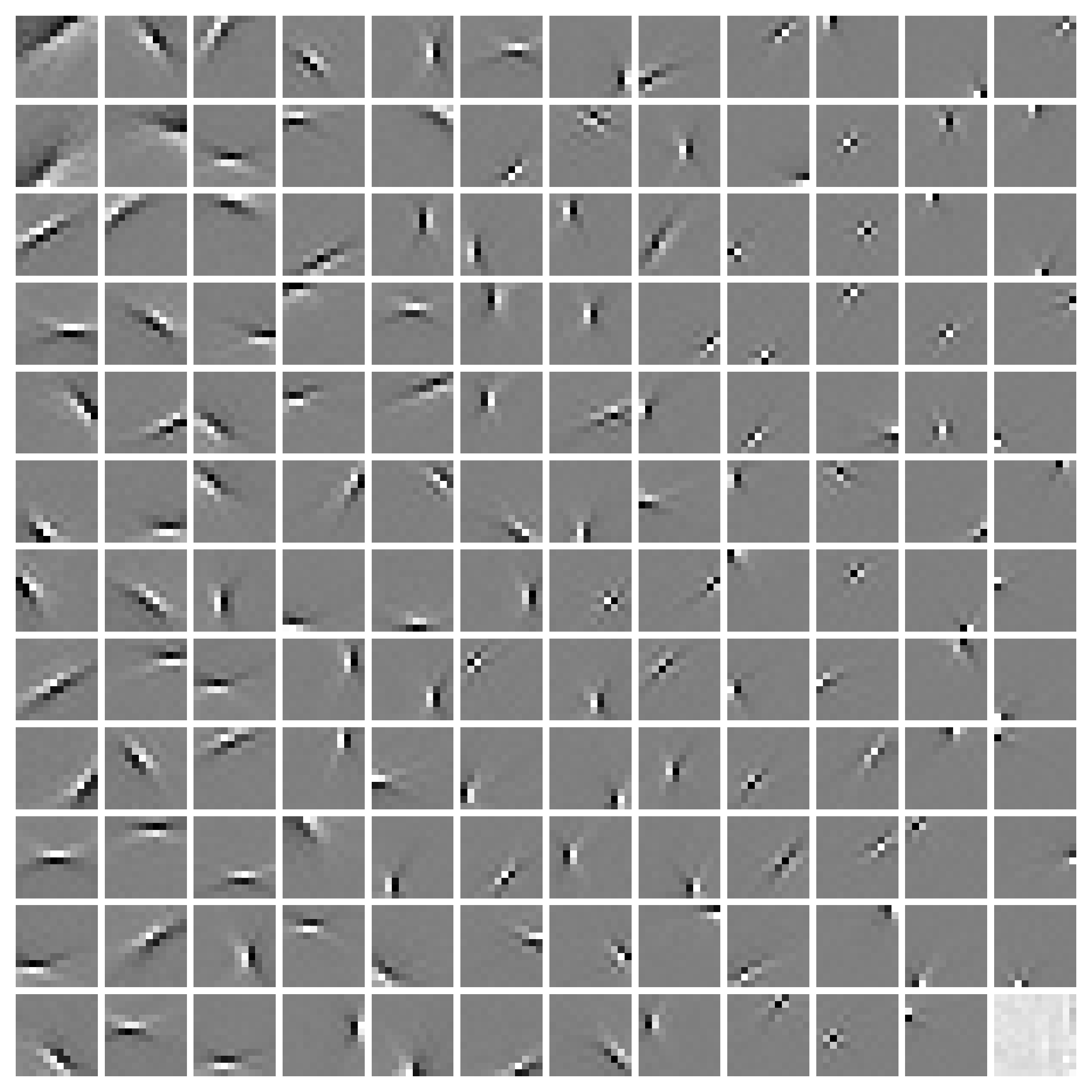}} \hspace{0pt}
\subfigure[]{\label{Fig1b}
\includegraphics[width= .315\linewidth]{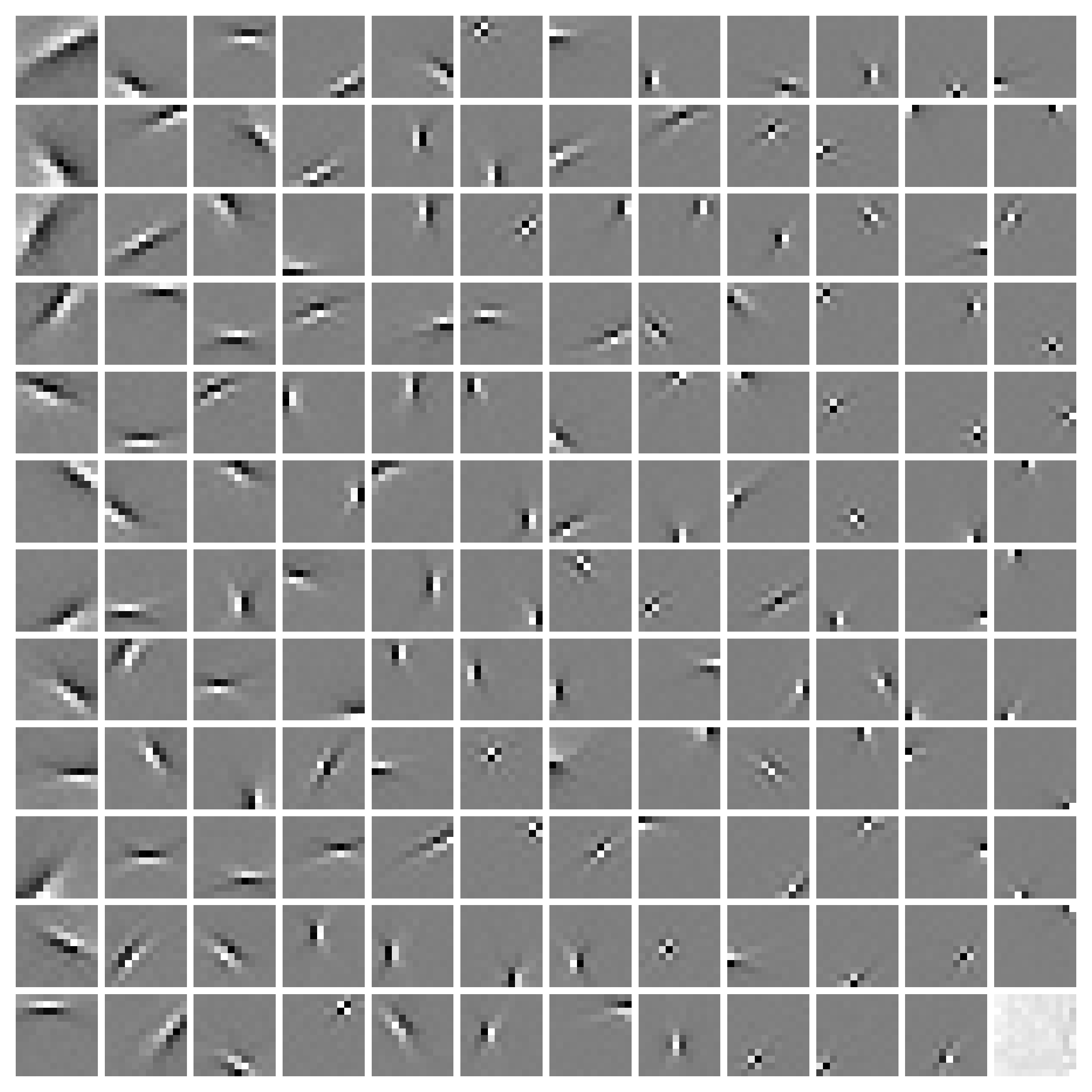}} \hspace{0pt}
\subfigure[]{\label{Fig1c}
\includegraphics[width= .315\linewidth]{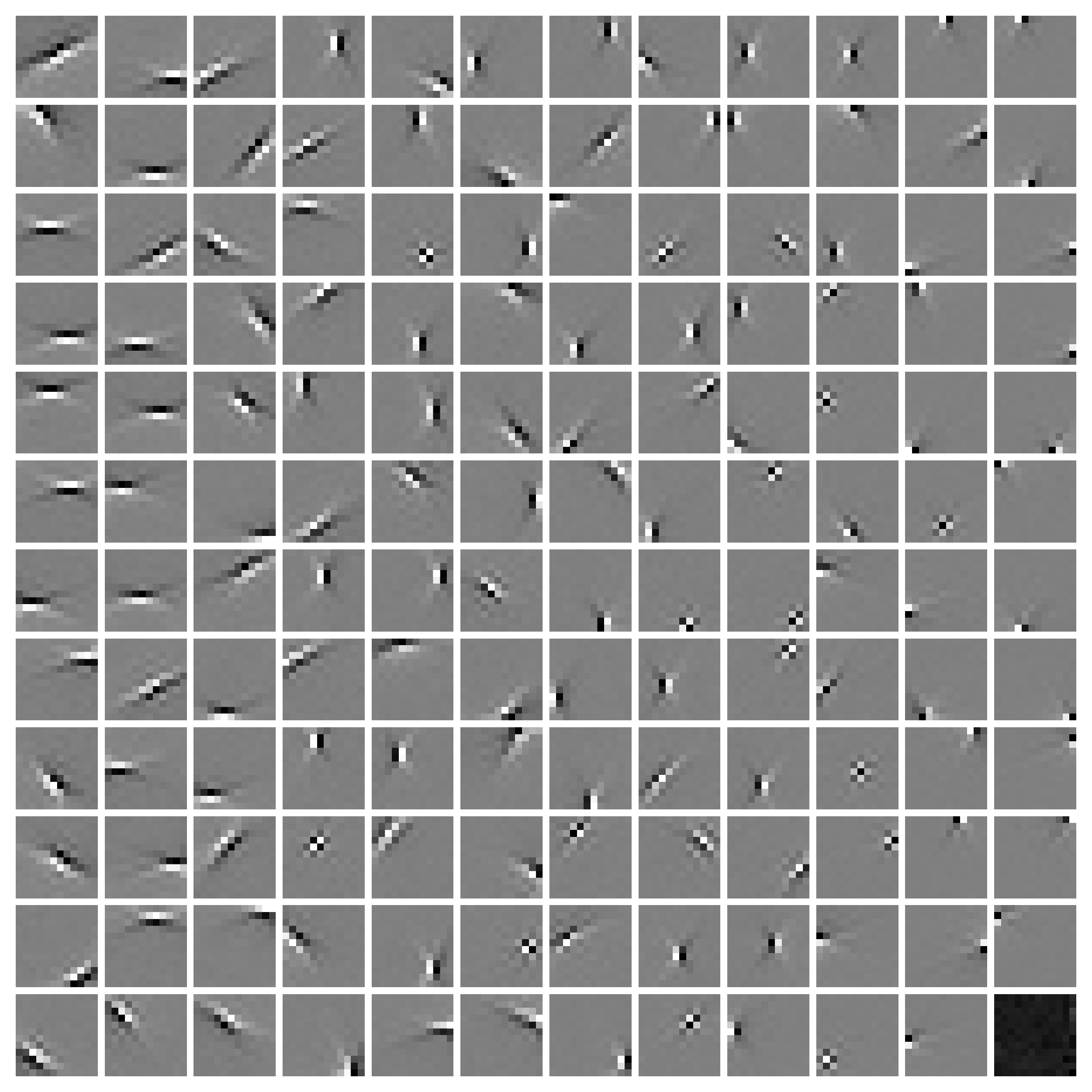}} \hspace{0pt}
\subfigure[]{\label{Fig1d}
\includegraphics[width= .315\linewidth]{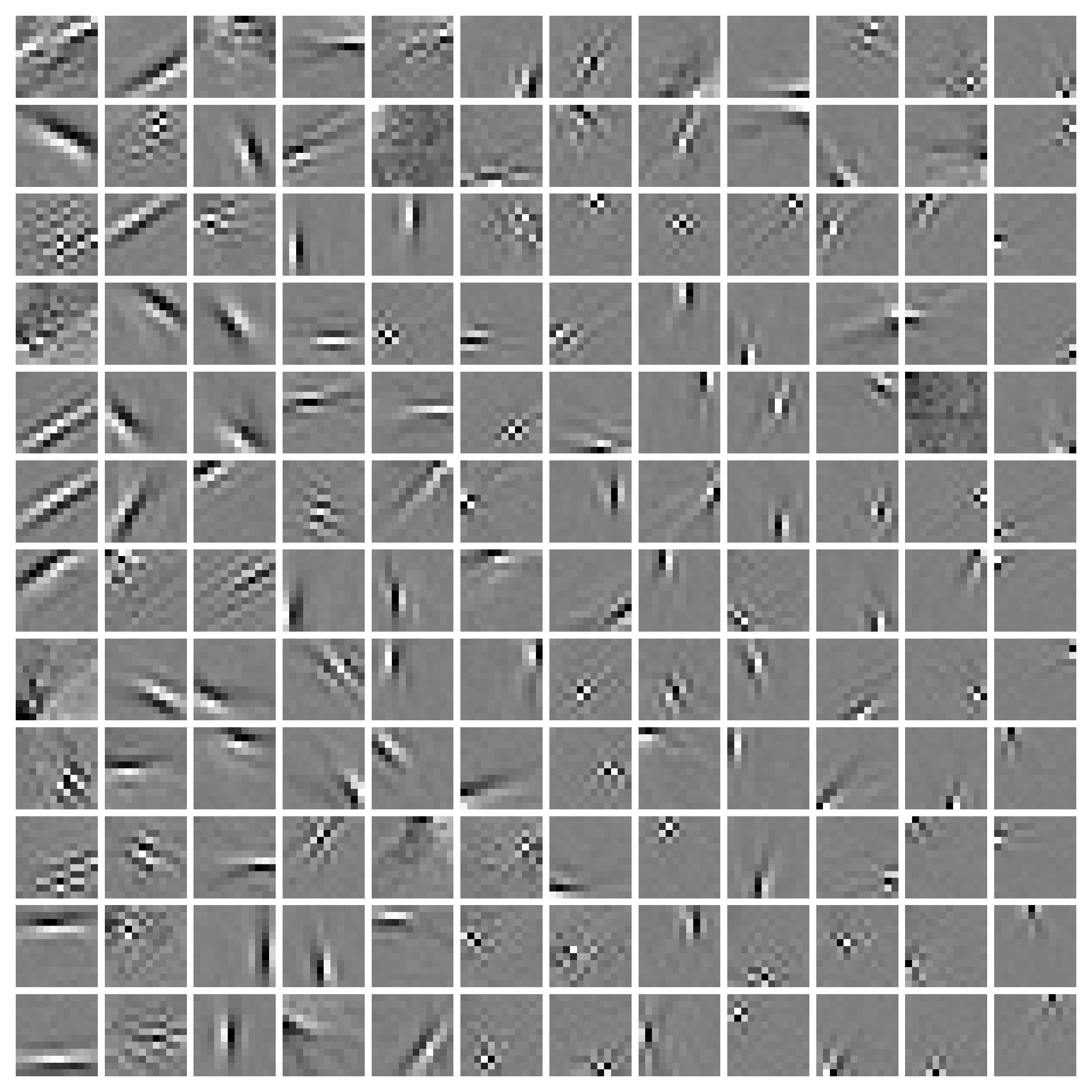}} \hspace{0pt}
\subfigure[]{\label{Fig1e}
\includegraphics[width= .315\linewidth]{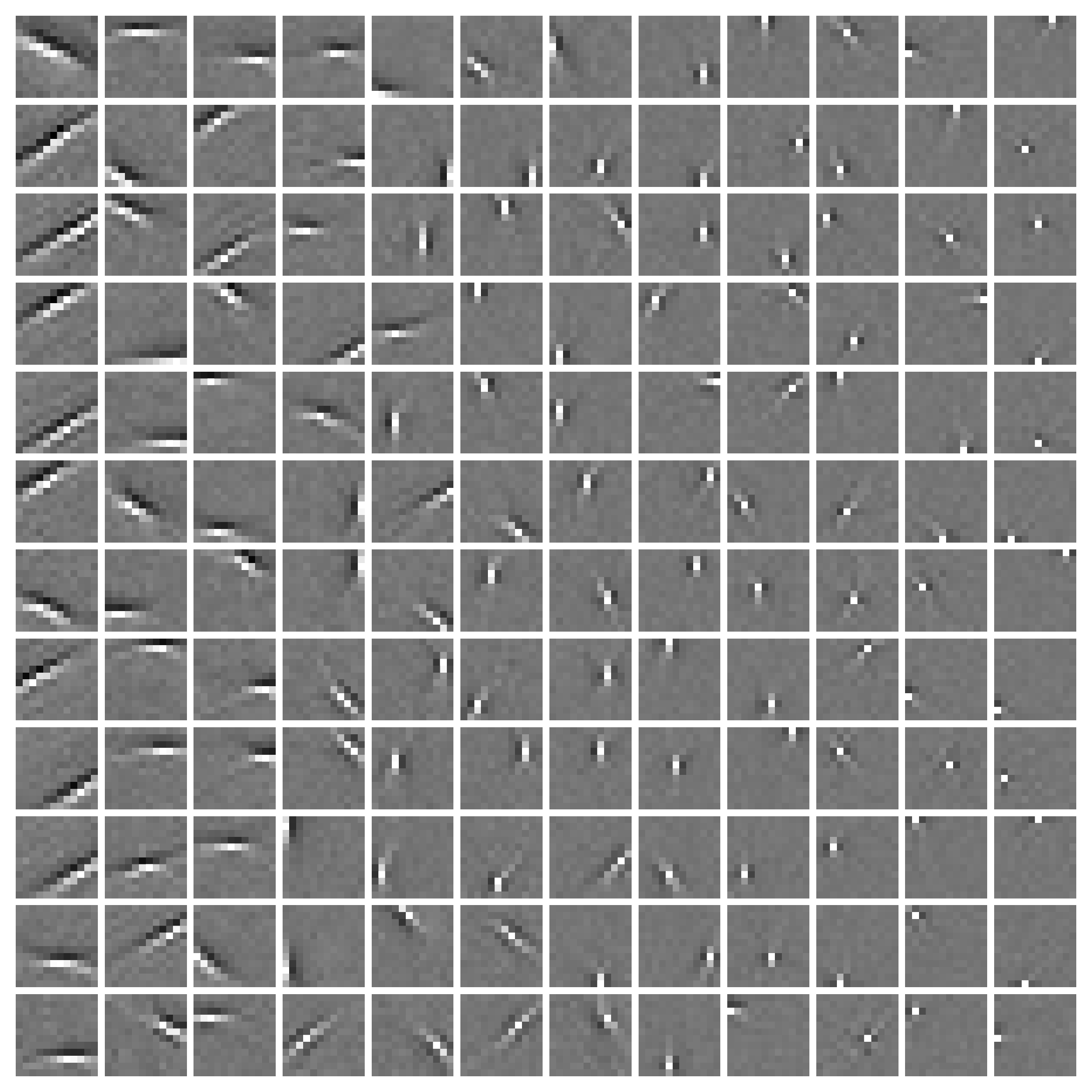}} \hspace{0pt}
\subfigure[]{\label{Fig1f}
\includegraphics[width= .315\linewidth]{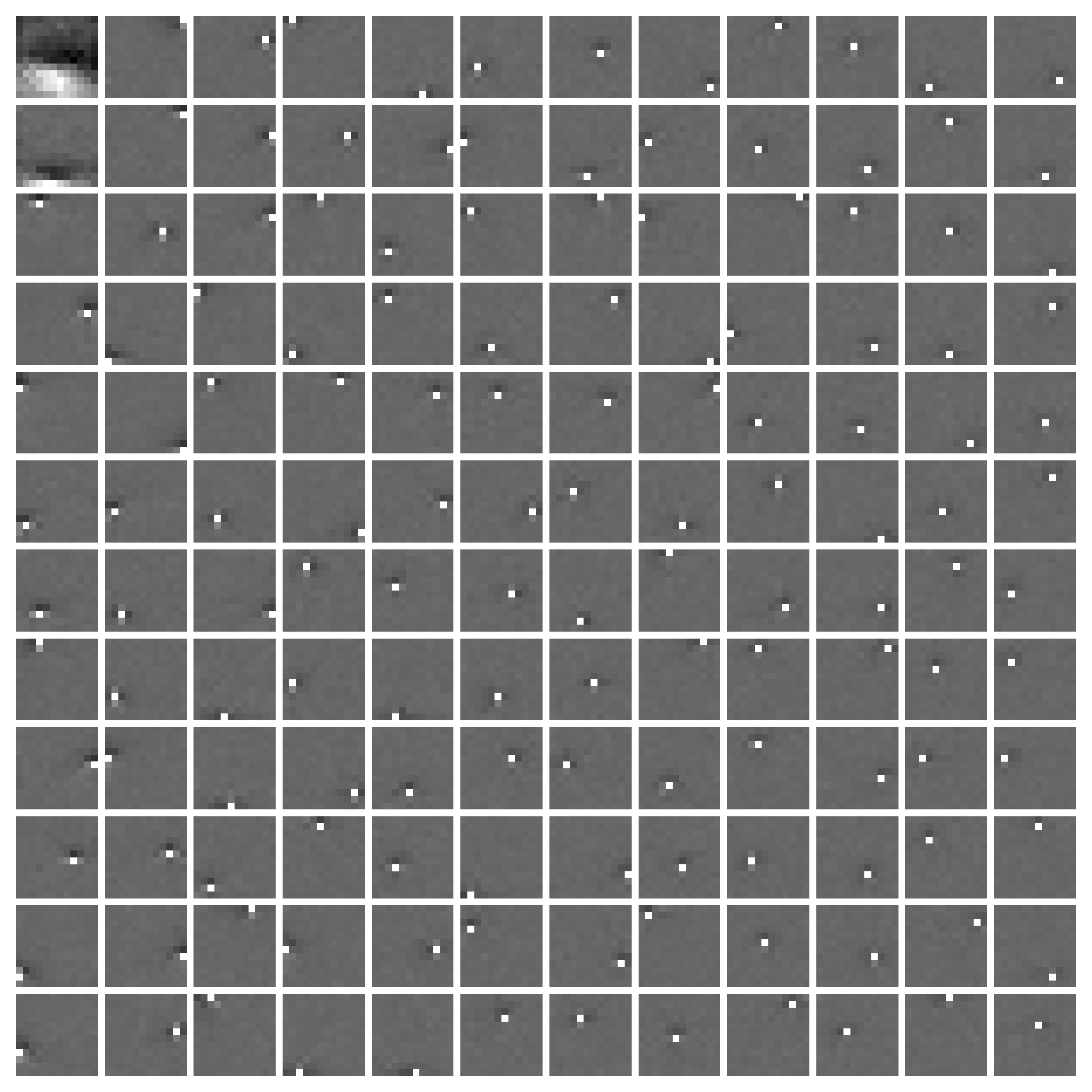}} \hspace{0pt}
\vskip -0.15 in\caption{Comparison of filters obtained from ${10^{5}}$ natural
image patches of size 12$\times$12 by our methods (Alg.1 and Alg.2) and other
methods. The number of output filters was $K_{1}=144$. (\textbf{a}): Alg.1.
(\textbf{b}): Alg.2. (\textbf{c}): IICA. (\textbf{d}): FICA. (\textbf{e}):
SRBM ($p=0.01$). (\textbf{f}): SRBM ($p=0.03$).}%
\label{Fig1}%
\end{figure}
\subfiglabelskip=0pt \begin{figure}[ptbh]
\vskip -0.2in \centering
\subfigure[]{\label{Fig2a}
\includegraphics[width= .315\linewidth]{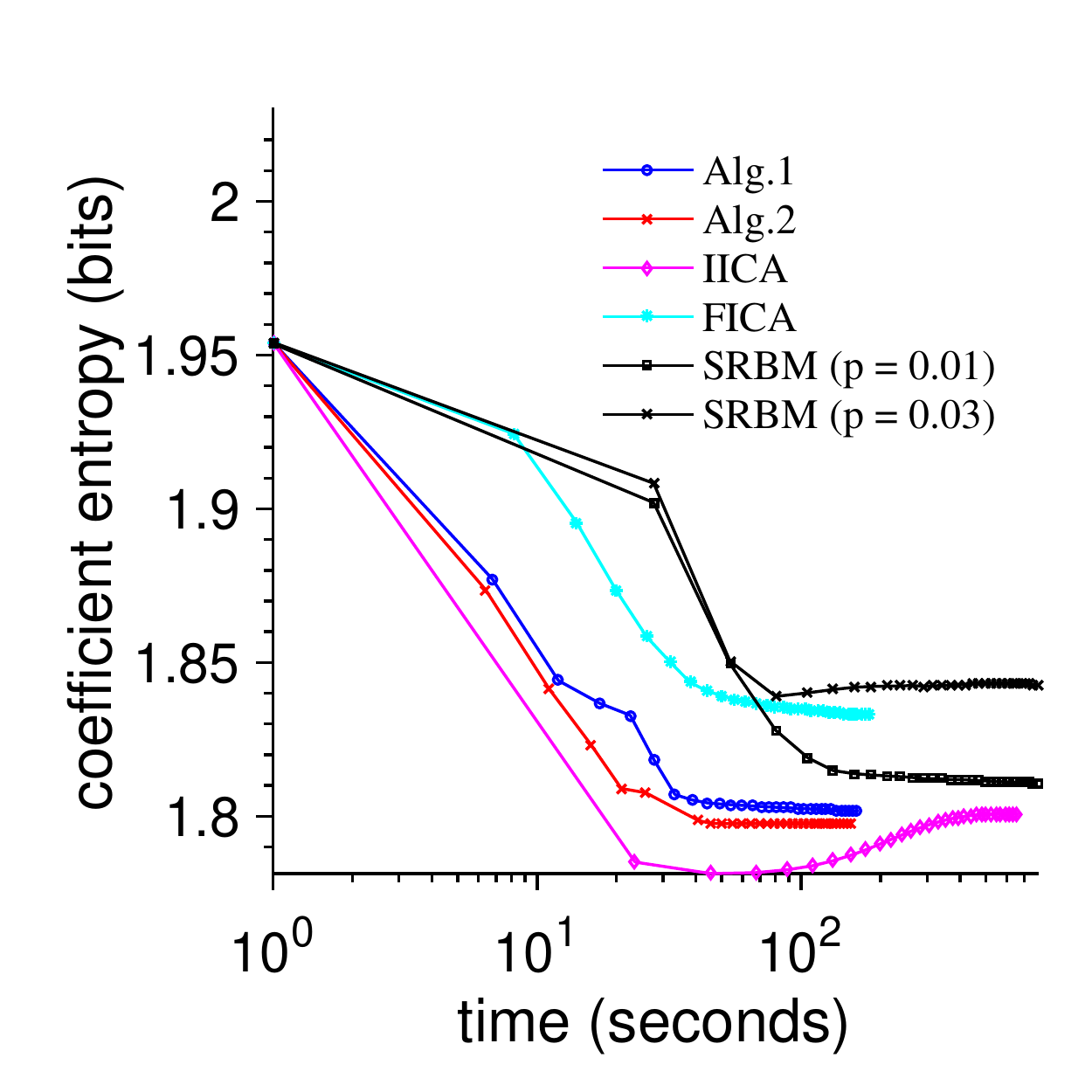}} \hspace{0pt}
\subfigure[]{\label{Fig2b}
\includegraphics[width= .315\linewidth]{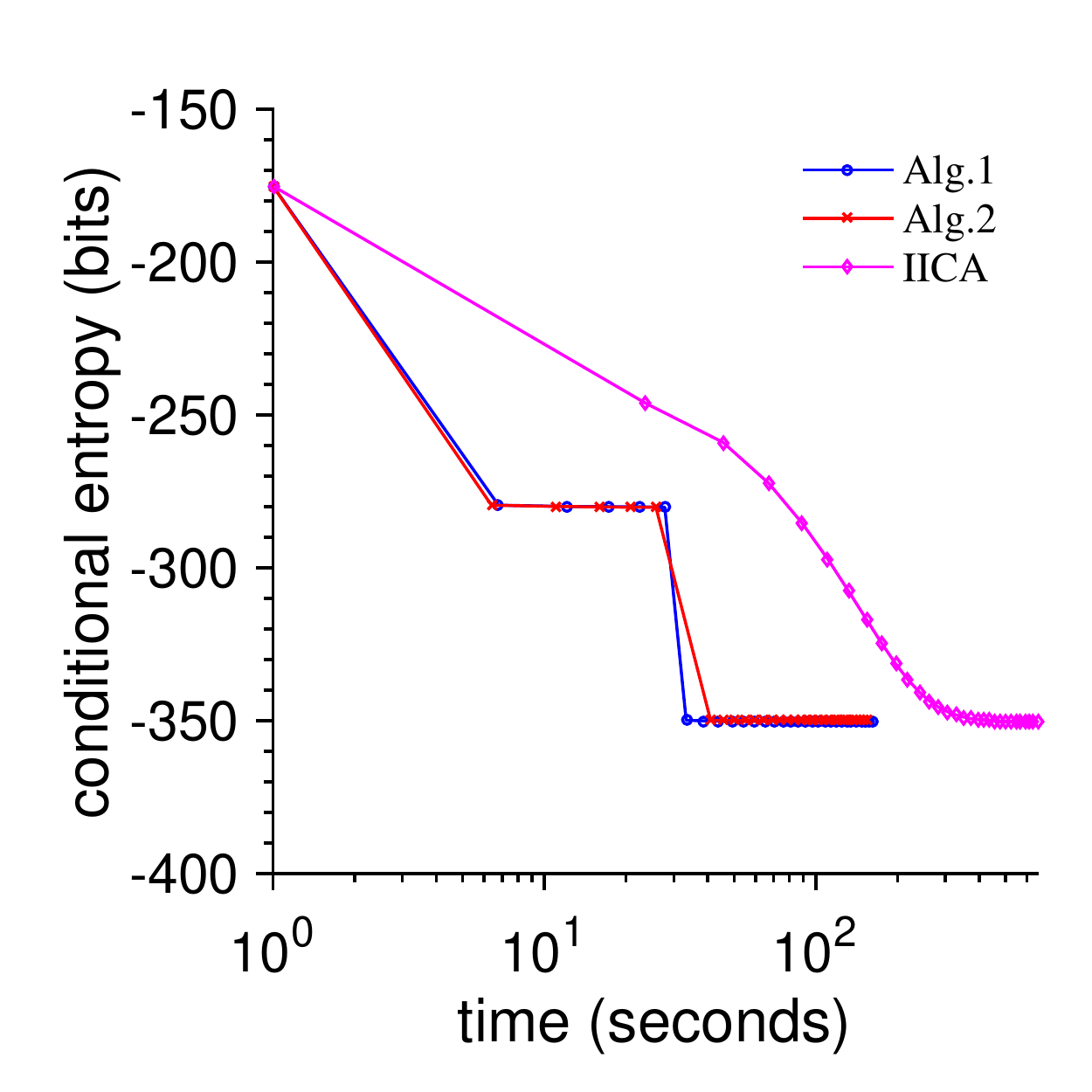}} \hspace{0pt}
\subfigure[]{\label{Fig2c}
\includegraphics[width= .315\linewidth]{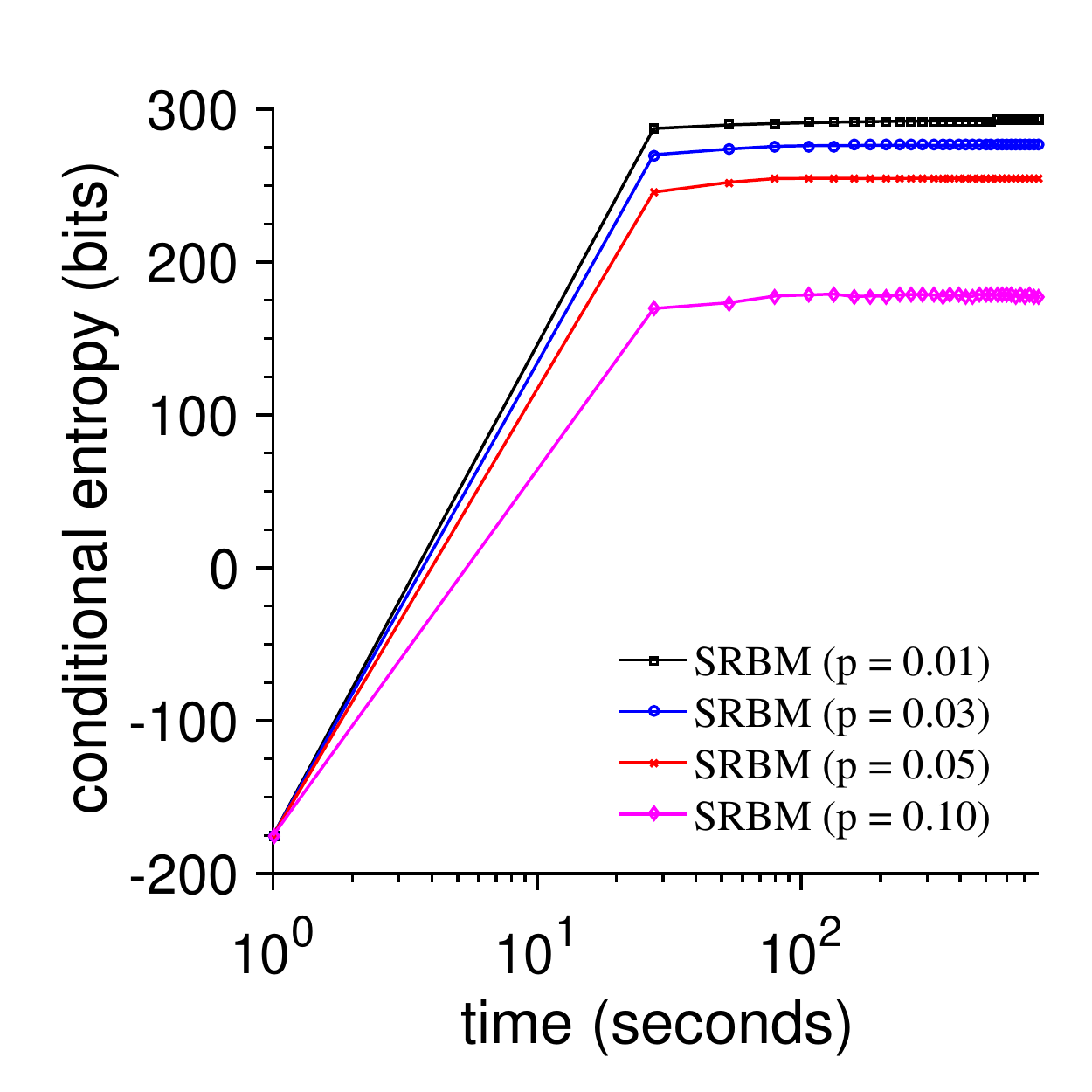}} \hspace{0pt}
\vskip -0.15in\caption{Comparison of quantization effects and convergence rate
by coefficient entropy (see \ref{h}) and conditional entropy (see \ref{h1})
corresponding to training results (filters) shown in Figure 2. The coefficient
entropy (panel \textbf{a}) and conditional entropy (panel \textbf{b} and
\textbf{c}) are shown as a function of training time on a logarithmic scale.
All experiments run on the same machine using Matlab. Here we sampled once
every $10$ epoches out of a total of $300$ epoches. We set epoch number
$t_{0}=50$ for Alg.1 and Alg.2 and the start time to $1$ second.}%
\label{Fig2}%
\end{figure}

Figure~\ref{Fig2} shows how the coefficient entropy (CFE) (see Eq. \ref{h})
and the conditional entropy (CDE) (see Eq. \ref{h1}) varied with training
time. We calculated CFE and CDE by sampling once every $10$ epoches from a
total of $300$ epoches. These results show that our algorithms had a fast
convergence rate towards stable solutions while having CFE and CDE values
similar to the algorithm of IICA, which converged much more slowly. Here the
values of CFE and CDE should be as small as possible for a good representation
learned from the same data set. Here we set epoch number $t_{0}=50$ in our
algorithms (see Alg.1 and Alg.2), and the start time was set to $1$ second.
This explains the step seen in Figure~\ref{Fig2} (b) for Alg.1 and Alg.2 since
the parameter $\beta$\ was updated when epoch number $t=t_{0}$. FICA had a
convergence rate close to our algorithms but had a big CFE, which is reflected
by the quality of the filter results in Figure~\ref{Fig1}. The convergence
rate and CFE for SRBM were close to IICA, but SRBM had a much bigger CDE than
IICA, which implies that the information had a greater loss when passing
through the system optimized by SRBM than by IICA or our methods.

From Figure~\ref{Fig2c} we see that the CDE (or MI $I(X;R)$, see Eq. \ref{IXR}
and \ref{h1}) decreases (or increases) with the increase of the value of the
sparseness control constant $p$. Note that a smaller $p$ means sparser
outputs. Hence, in this sense, increasing sparsity may result in sacrificing
some information. On the other hand, a weak sparsity constraint may lead to
failure of learning Gabor-like filters (see Figure~\ref{Fig1f}), and
increasing sparsity has an advantage in reducing the impact of noise in many
practical cases. Similar situation also occurs in sparse coding
\citep{Olshausen(1997-sparse)}, which provides a class of algorithms for
learning overcomplete dictionary representations of the input signals.
However, its training is time consuming due to its expensive computational
cost, although many new training algorithms have emerged
\citep[e.g.][]{Aharon(2006-k),Elad(2006-image),Lee(2006-IP-efficient),Mairal(2010-online)}.
See Appendix \ref{SupExp} for additional experimental results.

\section{Conclusions}

In this paper, we have presented a framework for unsupervised learning of
representations via information maximization for neural populations.
Information theory is a powerful tool for machine learning and it also
provides a benchmark of optimization principle for neural information
processing in nervous systems. Our framework is based on an asymptotic
approximation to MI for a large-scale neural population. To optimize the
infomax objective, we first use hierarchical infomax to obtain a good
approximation to the global optimal solution. Analytical solutions of the
hierarchical infomax are further improved by a fast convergence algorithm
based on gradient descent. This method allows us to optimize highly nonlinear
neural networks via hierarchical optimization using infomax principle.

From the viewpoint of information theory, the unsupervised pre-training for
deep learning \citep{Hinton(2006-reducing),Bengio(2007-greedy)} may be
reinterpreted as a process of hierarchical infomax, which might help explain
why unsupervised pre-training helps deep learning \citep{Erhan(2010-why)}. In
our framework, a pre-whitening step can emerge naturally by the hierarchical
infomax, which might also explain why a pre-whitening step is useful for
training in many learning algorithms \citep{Coates(2011-IP-analysis),Bengio(2012-deep)}.

Our model naturally incorporates a considerable degree of biological realism.
It allows the optimization of a large-scale neural population with noisy
spiking neurons while taking into account of multiple biological constraints,
such as membrane noise, limited energy, and bounded connection weights. We
employ a technique to attain a low-rank weight matrix for optimization, so as
to reduce the influence of noise and discourage overfitting during training.
In our model, many parameters are optimized, including the population density
of parameters, filter weight vectors, and parameters for nonlinear tuning
functions. Optimizing all these model parameters could not be easily done by
many other methods.

Our experimental results suggest that our method for unsupervised learning of
representations has obvious advantages in its training speed and robustness
over the main existing methods. Our model has a nonlinear feedforward
structure and is convenient for fast learning and inference. This simple and
flexible framework for unsupervised learning of presentations should be
readily extended to training deep structure networks. In future work, it would
interesting to use our method to train deep structure networks with either
unsupervised or supervised learning.

\section*{Acknowledgments}

\phantomsection\addcontentsline{toc}{section}{Acknowledgments}

We thank Prof. Honglak Lee for sharing Matlab code for algorithm comparison,
Prof. Shan Tan for discussions and comments and Kai Liu for helping draw
Figure 1. Supported by grant NIH-NIDCD R01 DC013698.

\noindent

\bibliographystyle{apalike2}
\bibliography{Ref_ICLR121816}

\appendix
\renewcommand{\theequation}{A.\arabic{equation}} \setcounter{equation}{0}
\renewcommand\thesubsection{A.\arabic{subsection}} \renewcommand\thesubsubsection{A.\arabic{subsection}.\arabic{subsubsection}}

\section*{Appendix}

\phantomsection\addcontentsline{toc}{section}{Appendix} \label{Appendix}

\subsection{Formulas for Approximation of Mutual Information}

\label{formulas}

It follows from $I(X;R)=\left\langle \ln\frac{p(\mathbf{x}|\mathbf{r}%
)}{p(\mathbf{x})}\right\rangle _{\mathbf{r},\mathbf{x}}$ and Eq. (\ref{Ia})
that the conditional entropy should read:
\begin{equation}
H(X|R)=-\left\langle \ln p(\mathbf{x}|\mathbf{r})\right\rangle _{\mathbf{r}%
,\mathbf{x}}\simeq-\frac{1}{2}\left\langle \ln\left(  \det\left(
\frac{\mathbf{G}(\mathbf{x})}{2\pi e}\right)  \right)  \right\rangle
_{\mathbf{x}}\text{.} \label{HXR}%
\end{equation}

The Fisher information matrix $\mathbf{J}(\mathbf{x})$ (see Eq. \ref{Jx}),
which is symmetric and positive semidefinite, can be written also as
\begin{equation}
\mathbf{J}(\mathbf{x})=\left\langle \frac{\partial\ln p(\mathbf{r}%
|\mathbf{x})}{\partial\mathbf{x}}\frac{\partial\ln p(\mathbf{r}|\mathbf{x}%
)}{\partial\mathbf{x}^{T}}\right\rangle _{\mathbf{r}|\mathbf{x}}\text{.}
\label{Jx1}%
\end{equation}

If we suppose $p(\mathbf{r}|\mathbf{x})$ is conditional independent, namely,
$p(\mathbf{r}|\mathbf{x})=\prod_{n=1}^{N}p(r_{n}|\mathbf{x};\boldsymbol{\theta
}_{n})$, then we have \citep[see ][]{Huang(2016-information)}
\begin{align}
&  \mathbf{J}(\mathbf{x})=N\int_{{\Theta}}p({\boldsymbol{\theta}}%
)\mathbf{S}(\mathbf{x};{\boldsymbol{\theta}})d{\boldsymbol{\theta}}%
\text{,}\label{Ppn1.1}\\
&  \mathbf{S}(\mathbf{x};\boldsymbol{\theta})=\left\langle \frac{\partial\ln
p(r|\mathbf{x};\boldsymbol{\theta})}{\partial\mathbf{x}}\frac{\partial\ln
p(r|\mathbf{x};\boldsymbol{\theta})}{\partial\mathbf{x}^{T}}\right\rangle
_{r|\mathbf{x}}\text{,} \label{Ppn1.2}%
\end{align}
where $p(\boldsymbol{\theta})$ is the population{ }density function of
parameter ${\boldsymbol{\theta}}$,
\begin{equation}
p(\boldsymbol{\theta})=\frac{1}{N}\sum_{n=1}^{N}\delta(\boldsymbol{\theta
}-\boldsymbol{\theta}_{n})\text{,} \label{p(theta)}%
\end{equation}
and $\delta(\cdot)$\ denotes the Dirac delta function. It can be proved that
the approximation function of MI $I_{G}[p({\boldsymbol{\theta}})]$ (Eq.
\ref{Ia}) is concave about $p({\boldsymbol{\theta}})$
\citep{Huang(2016-information)}. In Eq. (\ref{Ppn1.1}), we can approximate the
continuous integral by a discrete summation for numerical computation,%
\begin{equation}
\mathbf{J}(\mathbf{x})\approx N\sum_{k=1}^{K_{1}}\alpha_{k}\mathbf{S}%
(\mathbf{x};{\boldsymbol{\theta}}_{k})\text{,} \label{Ja}%
\end{equation}
where $\sum_{k=1}^{K_{1}}\alpha_{k}=1$, $\alpha_{k}>0$, $k=1,\cdots,K_{1}$,
$1\leq K_{1}\leq N$.

For Poisson neuron model, by Eq. (\ref{Ppn1.2}) we have
\citep[see ][]{Huang(2016-information)}
\begin{align}
{p(r|\mathbf{x}};{{\boldsymbol{\theta}})}  &  ={\dfrac{f(\mathbf{x}%
;{\boldsymbol{\theta}})^{r}}{r!}\exp\left(  -f(\mathbf{x};{\boldsymbol{\theta
}})\right)  }\text{,}\label{PoissNeuron.1}\\
{\mathbf{S}(\mathbf{x}};{{\boldsymbol{\theta}})}  &  ={\dfrac{1}%
{f(\mathbf{x};{\boldsymbol{\theta}})}\dfrac{\partial f(\mathbf{x}%
;{\boldsymbol{\theta}})}{\partial\mathbf{x}}\dfrac{\partial f(\mathbf{x}%
;{\boldsymbol{\theta}})}{\partial\mathbf{x}^{T}}}\nonumber\\
&  ={\dfrac{\partial g(\mathbf{x};{\boldsymbol{\theta}})}{\partial\mathbf{x}%
}\dfrac{\partial g(\mathbf{x};{\boldsymbol{\theta}})}{\partial\mathbf{x}^{T}}%
}\text{,} \label{PoissNeuron.2}%
\end{align}
where $f(\mathbf{x};{\boldsymbol{\theta}})\geq0$ is the activation function
(mean response) of neuron and
\begin{equation}
{g(\mathbf{x}};{{\boldsymbol{\theta}})}=2\sqrt{f(\mathbf{x}%
;{\boldsymbol{\theta}})}\text{.} \label{g(x)}%
\end{equation}
Similarly, for Gaussian noise model, we have%
\begin{align}
&  {p(r|\mathbf{x}};{{\boldsymbol{\theta}})}={\dfrac{1}{\sigma\sqrt{2\pi}}%
\exp\left(  -\dfrac{\left(  r-f(\mathbf{x};{\boldsymbol{\theta}})\right)
^{2}}{2\sigma^{2}}\right)  }\text{,}\label{GaussNeuron.1}\\
&  {\mathbf{S}(\mathbf{x}};{{\boldsymbol{\theta}})}={\dfrac{1}{\sigma^{2}%
}\dfrac{\partial f(\mathbf{x};{\boldsymbol{\theta}})}{\partial\mathbf{x}%
}\dfrac{\partial f(\mathbf{x};{\boldsymbol{\theta}})}{\partial\mathbf{x}^{T}}%
}\text{,} \label{GaussNeuron.2}%
\end{align}
where $\sigma>0$ denotes the standard deviation of noise.

Sometimes we do not know the specific form of $p(\mathbf{x})$ and only know
$M$ samples, $\mathbf{x}_{1}$, $\cdots$, $\mathbf{x}_{M}$, which are
independent and identically distributed (i.i.d.) samples drawn from the
distribution $p(\mathbf{x})$. Then we can use the empirical average to
approximate the integral in Eq. (\ref{Ia}):%
\begin{equation}
I_{G}\approx\frac{1}{2}\sum_{m=1}^{M}\ln\left(  \det\left(  \mathbf{G}%
(\mathbf{x}_{m})\right)  \right)  +H\left(  X\right)  \text{.} \label{Ia_1}%
\end{equation}

\subsection{Proof of Proposition \ref{Proposition 1}}

\label{Proof of Proposition 1}

\begin{proof}
It follows from the data-processing inequality \citep{Cover(2006-BK-elements)}
that%
\begin{align}
I(X;R)  &  \leq I(Y;R)\leq I(\breve{Y};R)\leq I(\bar{Y};R)\text{,
}\label{Ixr1}\\
I(X;R)  &  \leq I(X;\bar{Y})\leq I(X;\breve{Y})\leq I(X;Y)\text{.}
\label{Ixr2}%
\end{align}

Since
\begin{equation}
{p(\bar{y}_{k}|\mathbf{x}{)}}=p({\breve{y}_{k_{1}},\cdots,\breve{y}_{k_{N_{k}%
}}}|\mathbf{x)}={\mathcal{N}(\mathbf{w}_{k}^{T}\mathbf{x},\,N_{k}^{-1}%
\sigma^{2})}\text{, }k=1,\cdots,K_{1}\text{,} \label{proof_1}%
\end{equation}
we have%
\begin{align}
&  {p(\mathbf{\bar{y}}|\mathbf{x}{)}}=p(\mathbf{\breve{y}}|\mathbf{x)}%
\text{,}\label{proof_2}\\
&  {p(\mathbf{\bar{y}}{)}}=p(\mathbf{\breve{y})}\text{,}\label{proof_2a}\\
&  I(X;\bar{Y})=I(X;\breve{Y})\text{.} \label{proof_3}%
\end{align}
Hence, by (\ref{Ixr2}) and (\ref{proof_3}), expression (\ref{proposition_2}) holds.

On the other hand, when $N_{k}$ is large, from Eq. (\ref{Z-}) we know that the
distribution of $\bar{Z}_{k}$, namely, ${\mathcal{N}\left(  0\text{,\thinspace
}N_{k}^{-1}\sigma^{2}\right)  }$,${\ }$approaches a Dirac delta function
$\delta(\bar{z}_{k})$. Then by (\ref{y=wx}) and (\ref{vyk}) we have $p\left(
\mathbf{r}|\mathbf{\bar{y}}\right)  \simeq{p(\mathbf{r}|}${$\mathbf{y}$}${{)}%
}=p\left(  \mathbf{r}|\mathbf{x}\right)  $ and%
\begin{align}
I(X;R)  &  =I\left(  Y;R\right)  -\left\langle \ln\frac{p\left(
\mathbf{r}|\mathbf{y}\right)  }{p\left(  \mathbf{r}|\mathbf{x}\right)
}\right\rangle _{\mathbf{r},\mathbf{x}}=I\left(  Y;R\right)  \text{,}%
\label{Ixr3}\\
I\left(  Y;R\right)   &  =I(\bar{Y};R)-\left\langle \ln\frac{p\left(
\mathbf{r}|\mathbf{\bar{y}}\right)  }{p\left(  \mathbf{r}|\mathbf{y}\right)
}\right\rangle _{\mathbf{r},\mathbf{y},\mathbf{\bar{y}}}\simeq I(\bar
{Y};R)\text{,}\label{Ixr3a}\\
I\left(  Y;R\right)   &  =I(\breve{Y};R)-\left\langle \ln\frac{p\left(
\mathbf{r}|\mathbf{\breve{y}}\right)  }{p\left(  \mathbf{r}|\mathbf{y}\right)
}\right\rangle _{\mathbf{r},\mathbf{y},\mathbf{\breve{y}}}\simeq I(\breve
{Y};R)\text{,}\label{Ixr3b}\\
I(X;Y)  &  =I(X;\bar{Y})-\left\langle \ln\frac{p\left(  \mathbf{x}%
|\mathbf{\bar{y}}\right)  }{p\left(  \mathbf{x}|\mathbf{y}\right)
}\right\rangle _{\mathbf{x},\mathbf{y},\mathbf{\bar{y}}}\simeq I(X;\bar
{Y})\text{.} \label{Ixr3c}%
\end{align}
It follows from (\ref{Ixr1}) and (\ref{Ixr3}) that (\ref{proposition_1})
holds. Combining (\ref{proposition_1}), (\ref{proposition_2}) and
(\ref{Ixr3a})--(\ref{Ixr3c}), \ we immediately get (\ref{proposition_3a}) and
(\ref{proposition_3b}). This completes the proof of \textbf{Proposition
\ref{Proposition 1}}.
\end{proof}

\subsection{Hierarchical Optimization for Maximizing $I(X;R)$}

\label{A_HieOpt}

In the following, we will discuss the optimization procedure for maximizing
$I(X;R)$ in two stages: maximizing $I(X;\breve{Y})$ and maximizing $I(Y;R)$.

\subsubsection{The 1st Stage}

\label{1st}

In the first stage, our goal is to maximize the MI $I(X;\breve{Y})$ and get
the optimal parameters\textbf{ }$\mathbf{w}_{k}$ ($k=1,\cdots,K_{1}$). Assume
that the stimulus $\mathbf{x}$ has zero mean (if not, let $\mathbf{x}%
\leftarrow\mathbf{x}-\left\langle \mathbf{x}\right\rangle _{\mathbf{x}}$) and
covariance matrix $\boldsymbol{\Sigma}_{\mathbf{x}}$. It follows from
eigendecomposition that
\begin{equation}
\boldsymbol{\Sigma}_{\mathbf{x}}=\left\langle \mathbf{xx}^{T}\right\rangle
_{\mathbf{x}}\approx\frac{1}{M-1}\mathbf{XX}^{T}=\mathbf{U}\boldsymbol{\Sigma
}\mathbf{U}^{T}\text{,} \label{Sig}%
\end{equation}
where $\mathbf{X}=[\mathbf{x}_{1}$\textrm{,\thinspace}$\cdots$%
\textrm{,\thinspace}$\mathbf{x}_{M}]$, $\mathbf{U}=[\mathbf{u}_{1}%
,\cdots,\mathbf{u}_{K}]\in%
\mathbb{R}
^{K\times K}$ is an unitary orthogonal matrix and $\boldsymbol{\Sigma
}={\mathrm{diag}}\left(  \sigma_{1}^{2},\cdots,\sigma_{K}^{2}\right)  $ is a
positive diagonal matrix with $\sigma_{1}\geq\cdots\geq\sigma_{K}>0$. Define
\begin{align}
&  \mathbf{\tilde{x}}=\boldsymbol{\Sigma}^{-1/2}\mathbf{U}^{T}\mathbf{x}%
\text{,}\label{x1}\\
&  \mathbf{\tilde{w}}_{k}=\boldsymbol{\Sigma}{^{1/2}\mathbf{U}^{T}}%
\mathbf{w}_{k}\text{,}\label{x1a}\\
&  y_{k}=\mathbf{\tilde{w}}_{k}^{T}\mathbf{\tilde{x}}\text{,} \label{x1b}%
\end{align}
where $k=1,\cdots,K_{1}$. The covariance matrix of $\mathbf{\tilde{x}}$ is
given by
\begin{equation}
\boldsymbol{\Sigma}_{\mathbf{\tilde{x}}}=\left\langle \mathbf{\tilde{x}%
\tilde{x}}^{T}\right\rangle _{\mathbf{\tilde{x}}}\approx\mathbf{I}_{K}\text{,}
\label{sigI}%
\end{equation}
and $\mathbf{I}_{K}$ is a $K\times K$ identity matrix. From (\ref{Ia}) and
(\ref{GaussNeuron.2}) we have $I(X;\breve{Y})=I(\tilde{X};\breve{Y})$ and
\begin{align}
&  I(\tilde{X};\breve{Y})\simeq I_{G}^{\prime}=\frac{1}{2}\ln\left(
\det\left(  \frac{\mathbf{\tilde{G}}}{2\pi e}\right)  \right)  +H(\tilde
{X})\text{,}\label{Ia1}\\
&  \mathbf{\tilde{G}}\approx{N\sigma^{-2}\sum_{k=1}^{K_{1}}\alpha_{k}%
}\mathbf{\tilde{w}}_{k}\mathbf{\tilde{w}}_{k}^{T}+\mathbf{I}_{K}\text{.}
\label{G1}%
\end{align}
The following approximations are useful
\citep[see ][]{Huang(2016-information)}:
\begin{align}
p(\mathbf{\tilde{x}})  &  \approx{\mathcal{N}\left(  0,\mathbf{I}_{K}\right)
}\text{, }\label{Px1}\\
\mathbf{P}(\mathbf{\tilde{x}})  &  =-\frac{\partial^{2}\ln p\left(
\mathbf{\tilde{x}}\right)  }{\partial\mathbf{\tilde{x}}\partial\mathbf{\tilde
{x}}^{T}}\approx\mathbf{I}_{K}\text{.} \label{Px1a}%
\end{align}
By the central limit theorem, the distribution of random variable $\tilde{X}$
is closer to a normal distribution than the distribution of the original
random variable $X$. On the other hand, the PCA models assume multivariate
gaussian data whereas the ICA models assume multivariate non-gaussian data.
Hence by a PCA-like whitening transformation (\ref{x1}) we can use the
approximation (\ref{Px1a}) with the Laplace's method of asymptotic expansion,
which only requires that the peak be close to its mean while random variable
$\tilde{X}$\ needs not be exactly Gaussian.

Without any constraints on the Gaussian channel of neural populations,
especially the peak firing rates, the capacity of this channel may grow
indefinitely: $I(\tilde{X};\breve{Y})\rightarrow\infty$. The most common
constraint on the neural populations is an energy or power constraint which
can also be regarded as a signal-to-noise ratio (SNR) constraint. The SNR for
the output $\breve{y}_{n}$ of the \textit{n}-th neuron is given by%
\begin{equation}
\mathrm{SNR}_{n}=\frac{1}{\sigma^{2}}\left\langle \left(  \mathbf{w}_{n}%
^{T}\mathbf{x}\right)  ^{2}\right\rangle _{\mathbf{x}}\approx\frac{1}%
{\sigma^{2}}\mathbf{\tilde{w}}_{n}^{T}\mathbf{\tilde{w}}_{n}\text{,
}n=1,\cdots,N\text{.} \label{SNRn}%
\end{equation}
We require that%
\begin{equation}
\frac{1}{N}{\sum_{n=1}^{N}}\mathrm{SNR}_{n}\approx\frac{1}{\sigma^{2}}%
{\sum_{k=1}^{K_{1}}\alpha_{k}}\mathbf{\tilde{w}}_{k}^{T}\mathbf{\tilde{w}}%
_{k}\leq\rho\text{,} \label{cons.1}%
\end{equation}
where $\rho$ is a positive constant. Then by Eq. (\ref{Ia1}), (\ref{G1}) and
(\ref{cons.1}), we have the following optimization problem:
\begin{align}
&  \mathsf{minimize}\mathrm{\;}{Q_{G}^{\prime}[\mathbf{\hat{W}}]}=-\frac{1}%
{2}{\ln\left(  \det\left(  N\sigma^{-2}\mathbf{\hat{W}\hat{W}}^{T}%
+\mathbf{I}_{K}\right)  \right)  }\text{,}\label{minQa1}\\
&  \mathsf{subject\;to\;}{h}=\mathrm{Tr}\left(  \mathbf{\hat{W}\hat{W}}%
^{T}\right)  -E\leq0\text{,} \label{cons.h}%
\end{align}
where $\mathrm{Tr}\left(  \cdot\right)  $ denotes matrix trace and
\begin{align}
&  \mathbf{\hat{W}}=\mathbf{\tilde{W}A}^{1/2}={\boldsymbol{\Sigma}%
^{1/2}\mathbf{U}^{T}}\mathbf{WA}^{1/2}={\left[  \mathbf{\hat{w}}_{1}%
,\cdots,\mathbf{\hat{w}}_{K_{1}}\right]  }\text{,}\label{WA1}\\
&  {\mathbf{A}}={\mathrm{diag}\left(  \alpha_{1},\cdots,\alpha_{K_{1}}\right)
}\text{,}\label{A0}\\
&  \mathbf{W}={\left[  \mathbf{w}_{1},\cdots,\mathbf{w}_{K_{1}}\right]
}\text{,}\label{W}\\
&  \mathbf{\tilde{W}}={\left[  \mathbf{\tilde{w}}_{1},\cdots,\mathbf{\tilde
{w}}_{K_{1}}\right]  }\text{,}\label{Wt}\\
&  E=\rho\sigma^{2}\text{.} \label{E}%
\end{align}
Here $E$ is a constant that does not affect the final optimal solution so we
set $E=1$. Then we obtain an optimal solution as follows:
\begin{align}
\mathbf{W}  &  ={a}\mathbf{U}_{0}{\boldsymbol{\Sigma}_{0}^{-1/2}\mathbf{V}%
_{0}^{T}}\text{, }\label{W0}\\
\mathbf{A}  &  =K_{1}^{-1}{\mathbf{I}}_{K_{1}}\text{, }\label{A}\\
a  &  =\sqrt{E{K_{1}}K_{0}^{-1}}=\sqrt{{K_{1}}K_{0}^{-1}}\text{,}\label{a}\\
\boldsymbol{\Sigma}_{0}  &  ={\mathrm{diag}}\left(  \sigma_{1}^{2}%
,\cdots,\sigma_{K_{0}}^{2}\right)  \text{,}\label{sigma0}\\
\mathbf{U}_{0}  &  =\mathbf{U}\left(  \text{:},1\text{:}K_{0}\right)  \in%
\mathbb{R}
^{K\times K_{0}}\text{,}\label{U0}\\
\mathbf{V}_{0}  &  =\mathbf{V}\left(  \text{:},1\text{:}K_{0}\right)  \in%
\mathbb{R}
^{K_{1}\times K_{0}}\text{,} \label{V0}%
\end{align}
where $\mathbf{V}=[\mathbf{v}_{1},\cdots,\mathbf{v}_{K_{1}}]$ is an
$K_{1}\times K_{1}$\ unitary orthogonal matrix, parameter $K_{0}$ represents
the size of the reduced dimension ($1\leq K_{0}\leq K$), and its value will be
determined below. Now the optimal parameters $\mathbf{w}_{n}$ ($n=1,\cdots,N$)
are clustered into $K_{1}$\ classes (see Eq. \ref{Ja}) and obey an uniform
discrete distribution (see also Eq. \ref{alpha1} in Appendix \ref{2nd}).

When ${K}=K_{0}={K}_{1}$, the optimal solution of $\mathbf{W}$ in
Eq.~(\ref{W0}) is a whitening-like filter. When ${\mathbf{V=I}}_{K}$, the
optimal matrix $\mathbf{W}$ is the principal component analysis (PCA)
whitening filters. In the symmetrical case with ${\mathbf{V=U}}$, the optimal
matrix $\mathbf{W}$ becomes a zero component analysis (ZCA) whitening filter.
If $K<{K}_{1}$, this case leads to an overcomplete solution, whereas when
$K_{0}\leq{K}_{1}<K$, the undercomplete solution arises. Since $K_{0}\leq
K_{1}$ and $K_{0}\leq K$,$\ {Q_{G}^{\prime}}$ achieves its minimum when
$K_{0}=K$. However, in practice other factors may prevent it from reaching
this minimum. For example, consider the average of squared weights,
\begin{equation}
\varsigma={\sum_{k=1}^{K_{1}}\alpha_{k}}\left\Vert \mathbf{w}_{k}\right\Vert
^{2}=\mathrm{Tr}\left(  \mathbf{WAW}^{T}\right)  =\frac{E}{K_{0}}\sum
_{k=1}^{K_{0}}{\sigma_{k}^{-2}}\text{,} \label{consw}%
\end{equation}
where $\left\Vert \mathbf{\cdot}\right\Vert $ denotes the Frobenius norm. The
value of $\varsigma$ is extremely large when any ${\sigma_{k}\ }$becomes
vanishingly small. For real neurons these weights of connection are not
allowed to be too large. Hence we impose a limitation on the weights:
$\varsigma\leq E_{1}$, where $E_{1}$ is a positive constant. This yields
another constraint on the objective function,
\begin{equation}
\tilde{h}=\frac{E}{K_{0}}\sum_{k=1}^{K_{0}}{\sigma_{k}^{-2}-E_{1}\leq
0}\text{.} \label{consw1}%
\end{equation}
From (\ref{cons.h}) and (\ref{consw1}) we get the optimal $K_{0}=\arg
\max_{\tilde{K}_{0}}\left(  E\tilde{K}_{0}^{-1}\sum_{k=1}^{\tilde{K}_{0}%
}{\sigma_{k}^{-2}}\right)  $. By this constraint, small values of ${\sigma
_{k}^{2}}$ will often result in $K_{0}<K$ and a low-rank matrix $\mathbf{W}$
(Eq. \ref{W0}).

On the other hand, the low-rank matrix $\mathbf{W}$ can filter out the noise
of stimulus $\mathbf{x}$. Consider the transformation $\mathbf{Y}%
=\mathbf{W}^{T}\mathbf{X}$ with $\mathbf{X}=[\mathbf{x}_{1}$%
\textrm{,\thinspace}$\cdots$\textrm{,\thinspace}$\mathbf{x}_{M}]$
and\ $\mathbf{Y}=[\mathbf{y}_{1}$\textrm{,\thinspace\thinspace}$\cdots
$\textrm{,\thinspace}$\mathbf{y}_{M}]$ for $M$ samples. It follows from the
singular value decomposition (SVD) of $\mathbf{X}$ that
\begin{equation}
\mathbf{X}=\mathbf{US\tilde{V}}^{T}\text{,} \label{Xsvd}%
\end{equation}
where $\mathbf{U}$ is given in (\ref{Sig}), $\mathbf{\tilde{V}}$ is a $M\times
M$ unitary orthogonal matrix, $\mathbf{S}\ $is a $K\times M$ diagonal matrix
with non-negative real numbers on the diagonal, $S_{k,k}=\sqrt{M-1}\sigma_{k}$
($k=1,\cdots,K$, $K\leq M$), and $\mathbf{SS}^{T}=(M-1){\boldsymbol{\Sigma}}$.
Let
\begin{equation}
\mathbf{\breve{X}}=\sqrt{M-1}\mathbf{U}_{0}{\boldsymbol{\Sigma}}_{0}%
^{1/2}\mathbf{\tilde{V}}_{0}^{T}\approx\mathbf{X}\text{,} \label{X(}%
\end{equation}
where $\mathbf{\tilde{V}}_{0}=\mathbf{\tilde{V}}\left(  \text{:}%
,1\text{:}K_{0}\right)  \in%
\mathbb{R}
^{M\times K_{0}}$, ${\boldsymbol{\Sigma}}_{0}$ and $\mathbf{U}_{0}$ are given
in (\ref{sigma0}) and (\ref{U0}), respectively. Then
\begin{equation}
\mathbf{Y}=\mathbf{W}^{T}\mathbf{X}={a\mathbf{V}_{0}\boldsymbol{\Sigma}%
_{0}^{-1/2}}\mathbf{U}_{0}^{T}\mathbf{US\tilde{V}}^{T}=\mathbf{W}%
^{T}\mathbf{\breve{X}}=a\sqrt{M-1}{\mathbf{V}_{0}}\mathbf{\tilde{V}}_{0}%
^{T}\text{,} \label{Y1}%
\end{equation}
where $\mathbf{\breve{X}}$\ can be regarded as a denoised version of
$\mathbf{X}$. The determination of the effective rank $K_{0}\leq K$ of the
matrix $\mathbf{\breve{X}}$ by using singular values is based on various
criteria \citep{Konstantinides(1988-statistical)}. Here we choose $K_{0}$ as
follows:
\begin{equation}
K_{0}=\arg\min_{K_{0}^{\prime}}\left(  \sqrt{\frac{\sum_{k=1}^{K_{0}^{\prime}%
}{\sigma_{k}^{2}}}{\sum_{k=1}^{K}{\sigma_{k}^{2}}}}\geq\epsilon\right)
\text{,} \label{K0}%
\end{equation}
where $\epsilon\ $is a positive constant ($0<\epsilon\leq1$).

Another advantage of a low-rank matrix $\mathbf{W}$ is that it can
significantly reduce overfitting for learning neural population{ }parameters.
In practice, the constraint (\ref{consw}) is equivalent to a weight-decay
regularization term used in many other optimization problems
\citep{Cortes(1995-support), Hinton(2010-practical)}, which can reduce
overfitting to the training data. To prevent the neural networks from
overfitting, \cite{Srivastava(2014-dropout)}\ presented a technique to
randomly drop units from the neural network during training, which may in fact
be regarded as an attempt to reduce the rank of the weight matrix because the
dropout can result in a sparser weights (lower rank matrix). This means that
the update is only concerned with keeping the more important components, which
is similar to first performing a denoising process by the SVD low rank approximation.

In this stage, we have obtained the optimal parameter $\mathbf{W}$ (see
\ref{W0}). The optimal value of matrix ${\mathbf{V}_{0}}$ can also be
determined, as shown in Appendix \ref{final}.

\subsubsection{The 2nd Stage}

\label{2nd}

For this stage, our goal is to maximize the MI $I(Y;R)$ and get the optimal
parameters ${\boldsymbol{\tilde{\theta}}}_{k}$, $k=1,\cdots,K_{1}$. Here the
input is $\mathbf{y}=(y_{1},\cdots,y_{K_{1}})^{T}$ and the output
$\mathbf{r}=(r_{1},\cdots,r_{N})^{T}$ is also clustered into $K_{1}$ classes.
The responses of $N_{k}$ neurons in the $k$-th subpopulation obey a Poisson
distribution with mean $\tilde{f}(\mathbf{e}_{k}^{T}\mathbf{y}%
;{\boldsymbol{\tilde{\theta}}}_{k})$, where $\mathbf{e}_{k}$ is a unit vector
with $1$ in the $k$-th\ element and $y_{k}=\mathbf{e}_{k}^{T}\mathbf{y}$. By
(\ref{x1}) and (\ref{x1b}), we have
\begin{align}
&  \left\langle y_{k}\right\rangle _{y_{k}}=0\text{, }\label{yv0}\\
&  \sigma_{y_{k}}^{2}=\left\langle y_{k}^{2}\right\rangle _{y_{k}}=\left\Vert
\mathbf{\tilde{w}}_{k}\right\Vert ^{2}\text{.} \label{sigyk2}%
\end{align}
Then for large $N$, by (\ref{Ia})--(\ref{Px}) and (\ref{Px1}) we can use the
following approximation,
\begin{equation}
I(Y;R)\simeq\breve{I}_{F}=\frac{1}{2}\left\langle \ln\left(  \det\left(
\frac{\mathbf{\breve{J}}(\mathbf{y})}{2\pi e}\right)  \right)  \right\rangle
_{\mathbf{y}}+H(Y)\text{,} \label{Ia2}%
\end{equation}
where
\begin{align}
&  \mathbf{\breve{J}}(\mathbf{y})={{\mathrm{diag}}}\left(  N\alpha
_{1}\left\vert g_{1}^{\prime}(y_{1})\right\vert ^{2},\cdots,N\alpha_{K_{1}%
}\left\vert g_{K_{1}}^{\prime}(y_{K_{1}})\right\vert ^{2}\right)
\text{,}\label{3.1}\\
&  g_{k}^{\prime}(y_{k})=\frac{\partial g_{k}(y_{k})}{\partial y_{k}}\text{,
}k=1,\cdots,K_{1}\text{,}\label{3.3}\\
&  g_{k}(y_{k})=2\sqrt{\tilde{f}(y_{k};{\boldsymbol{\tilde{\theta}}}_{k}%
)}\text{, }k=1,\cdots,K_{1}\text{.} \label{3.4}%
\end{align}

It is easy to get that%
\begin{align}
\breve{I}_{F}  &  =\frac{1}{2}\sum_{k=1}^{K_{1}}\left\langle \ln\left(
\frac{N\alpha_{k}\left\vert g_{k}^{\prime}(y_{k})\right\vert ^{2}}{2\pi
e}\right)  \right\rangle _{\mathbf{y}}+H(Y)\nonumber\\
&  \leq\frac{1}{2}\sum_{k=1}^{K_{1}}\left\langle \ln\left(  \frac{\left\vert
g_{k}^{\prime}(y_{k})\right\vert ^{2}}{2\pi e}\right)  \right\rangle
_{\mathbf{y}}-\frac{K_{1}}{2}\ln\left(  \frac{K_{1}}{N}\right)  +H(Y)\text{,}
\label{IF<}%
\end{align}
where the equality holds if and only if
\begin{equation}
\alpha_{k}=\frac{1}{K_{1}},k=1,\cdots,K_{1}\text{,} \label{alpha1}%
\end{equation}
which is consistent with Eq.\ (\ref{A}).

On the other hand, it follows from the Jensen's inequality that%
\begin{align}
\breve{I}_{F}  &  =\left\langle \ln\left(  p\left(  \mathbf{y}\right)
^{-1}\det\left(  \frac{\mathbf{\breve{J}}(\mathbf{y})}{2\pi e}\right)
^{1/2}\right)  \right\rangle _{\mathbf{y}}\nonumber\\
&  \leq\ln\int\det\left(  \dfrac{\mathbf{\breve{J}}(\mathbf{y})}{2\pi
e}\right)  ^{1/2}d\mathbf{y}\text{,} \label{3.6}%
\end{align}
where the equality holds if and only if $p\left(  \mathbf{y}\right)  ^{-1}%
\det\left(  \mathbf{\breve{J}}(\mathbf{y})\right)  ^{1/2}$ is a constant,
which implies that%
\begin{equation}
p\left(  \mathbf{y}\right)  =\frac{\det\left(  \mathbf{\breve{J}}%
(\mathbf{y})\right)  ^{1/2}}{\int\det\left(  \mathbf{\breve{J}}(\mathbf{y}%
)\right)  ^{1/2}d\mathbf{y}}=\frac{%
{\textstyle\prod_{k=1}^{K_{1}}}
\left\vert g_{k}^{\prime}(y_{k})\right\vert }{\int%
{\textstyle\prod_{k=1}^{K_{1}}}
\left\vert g_{k}^{\prime}(y_{k})\right\vert d\mathbf{y}}\text{.} \label{3.7}%
\end{equation}
From (\ref{3.6}) and (\ref{3.7}), maximizing $\tilde{I}_{F}$\ yields
\begin{equation}
p\left(  y_{k}\right)  =\frac{\left\vert g_{k}^{\prime}(y_{k})\right\vert
}{\int\left\vert g_{k}^{\prime}(y_{k})\right\vert dy_{k}}\text{, }%
k=1,\cdots,K_{1}\text{.} \label{py}%
\end{equation}
We assume that (\ref{py}) holds, at least approximately. Hence we can let the
peak of $g_{k}^{\prime}(y_{k})$ be at $y_{k}=\left\langle y_{k}\right\rangle
_{y_{k}}=0$ and $\left\langle y_{k}^{2}\right\rangle _{y_{k}}=\sigma_{y_{k}%
}^{2}=\left\Vert \mathbf{\tilde{w}}_{k}\right\Vert ^{2}$. Then combining
(\ref{3.3}), (\ref{3.6}) and (\ref{py}) we find the optimal parameters
${\boldsymbol{\tilde{\theta}}}_{k}$ for the nonlinear functions $\tilde
{f}(y_{k};{\boldsymbol{\tilde{\theta}}}_{k})$, $k=1,\cdots,K_{1}$.

\subsubsection{The Final Objective Function}

\label{final}

In the preceding sections we have obtained the initial optimal solutions by
maximizing $I\left(  X;\breve{Y}\right)  $\ and $I(Y;R)$. In this section, we
will discuss how to find the final optimal $\mathbf{V}_{0}$ and other
parameters by maximizing $I(X;R)$ from the initial optimal solutions.

First, we have
\begin{equation}
\mathbf{y}=\mathbf{\tilde{W}}^{T}\mathbf{\tilde{x}}=a\mathbf{\hat{y}}\text{,}
\label{ya}%
\end{equation}
where $a$ is given in (\ref{a}) and
\begin{align}
&  \mathbf{\hat{y}}=(\hat{y}_{1},\cdots,\hat{y}_{K_{1}})^{T}={\mathbf{C}}%
^{T}\mathbf{\hat{x}}={\mathbf{\check{C}}}^{T}\mathbf{\check{x}}\text{,}%
\label{y^}\\
&  \mathbf{\hat{x}}=\mathbf{\Sigma}_{0}^{-1/2}\mathbf{U}_{0}^{T}%
\mathbf{x}\text{,}\label{x^}\\
&  {\mathbf{C}}={\mathbf{V}_{0}^{T}}\in%
\mathbb{R}
^{K_{0}\times K_{1}}\text{,}\label{C}\\
&  \mathbf{\check{x}}=\mathbf{U}_{0}\mathbf{\Sigma}_{0}^{-1/2}\mathbf{U}%
_{0}^{T}\mathbf{x=\mathbf{U}_{0}\hat{x}}\text{, }\label{xv}\\
&  {\mathbf{\check{C}}}=\mathbf{U}_{0}{\mathbf{C}}=[\mathbf{\check{c}}%
_{1},\cdots,\mathbf{\check{c}}_{K_{1}}]\text{.} \label{Cv}%
\end{align}

It follows that%
\begin{align}
&  I(X;R)=I\left(  \tilde{X};R\right)  \simeq\tilde{I}_{G}=\frac{1}%
{2}\left\langle \ln\left(  \det\left(  \frac{\mathbf{G}(\mathbf{\hat{x}}%
)}{2\pi e}\right)  \right)  \right\rangle _{\mathbf{\hat{x}}}+H(\tilde
{X})\text{,}\label{Ia3}\\
&  \mathbf{G}(\mathbf{\hat{x}})=N\mathbf{\hat{W}\boldsymbol{\hat{\Phi}}\hat
{W}}^{T}+\mathbf{I}_{K}\text{,}\label{Gx1}\\
&  \mathbf{\hat{W}}={\boldsymbol{\Sigma}^{1/2}\mathbf{U}^{T}}\mathbf{WA}%
^{1/2}=a\sqrt{K_{1}^{-1}}\mathbf{I}_{K_{0}}^{K}{\mathbf{C}}=\sqrt{K_{0}^{-1}%
}\mathbf{I}_{K_{0}}^{K}{\mathbf{C}}\text{,} \label{W2}%
\end{align}
where $\mathbf{I}_{K_{0}}^{K}$ is a $K\times K_{0}$\ diagonal matrix with
value $1$ on the diagonal and
\begin{align}
&  \boldsymbol{\hat{\Phi}}=\boldsymbol{\Phi}^{2}\text{,}\label{final.4}\\
&  \boldsymbol{\Phi}={{\mathrm{diag}}}\left(  \phi(\hat{y}_{1}),\cdots
,\phi(\hat{y}_{K_{1}})\right)  \text{,}\label{final.5}\\
&  \phi(\hat{y}_{k})=a^{-1}\left\vert \frac{\partial g_{k}(\hat{y}_{k}%
)}{\partial\hat{y}_{k}}\right\vert \text{,}\label{final.6}\\
&  g_{k}(\hat{y}_{k})=2\sqrt{\tilde{f}(\hat{y}_{k};{\boldsymbol{\tilde{\theta
}}}_{k})}\text{,}\label{final.7}\\
&  \hat{y}_{k}=a^{-1}y_{k}=\mathbf{c}_{k}^{T}\mathbf{\hat{x}}\text{,
}k=1,\cdots,K_{1}\text{.} \label{final.8}%
\end{align}
Then we have
\begin{equation}
\det\left(  \mathbf{G}(\mathbf{\hat{x}})\right)  =\det\left(  NK_{0}%
^{-1}\mathbf{C\boldsymbol{\hat{\Phi}}C}^{T}+\mathbf{I}_{K_{0}}\right)
\text{.} \label{Gx2}%
\end{equation}
For large $N$ and $K_{0}/N\rightarrow0$, we have
\begin{equation}
\det\left(  \mathbf{G}(\mathbf{\hat{x}})\right)  \approx\det\left(
\mathbf{J}(\mathbf{\hat{x}})\right)  =\det\left(  NK_{0}^{-1}%
\mathbf{C\boldsymbol{\hat{\Phi}}C}^{T}\right)  \text{,} \label{Gx3}%
\end{equation}
\begin{align}
&  \tilde{I}_{G}\approx\tilde{I}_{F}=-Q-\frac{K}{2}\ln\left(  2\pi e\right)
-\frac{K_{0}}{2}\ln\left(  \varepsilon\right)  +H(\tilde{X})\text{,}%
\label{Ia4}\\
&  Q=-\frac{1}{2}\left\langle \ln\left(  \det\left(
{\mathbf{C\boldsymbol{\hat{\Phi}}C}}^{T}\right)  \right)  \right\rangle
_{\mathbf{\hat{x}}}\text{,}\label{Q1}\\
&  \varepsilon=\frac{K_{0}}{N}\text{.} \label{eps1}%
\end{align}
Hence we can state the optimization problem as:%
\begin{align}
&  \mathsf{minimize}\;\;{Q\left[  {\mathbf{C}}\right]  =}-\frac{1}%
{2}\left\langle \ln\left(  \det\left(  {\mathbf{C\boldsymbol{\hat{\Phi}}C}%
}^{T}\right)  \right)  \right\rangle _{\mathbf{\hat{x}}}\text{,}\label{QC}\\
&  \mathsf{subject\;to}\;{\mathbf{CC}}^{T}=\mathbf{I}_{K_{0}}\text{.}
\label{CC}%
\end{align}
The gradient from (\ref{QC}) is given by:%
\begin{equation}
\frac{dQ{\left[  {\mathbf{C}}\right]  }}{d{\mathbf{C}}}=-\left\langle \left(
{\mathbf{C\boldsymbol{\hat{\Phi}}C}}^{T}\right)  ^{-1}\mathbf{C}%
\boldsymbol{\hat{\Phi}}+\mathbf{\hat{x}}\boldsymbol{\omega}^{T}\right\rangle
_{\mathbf{\hat{x}}}\text{,} \label{dQ}%
\end{equation}
where ${\mathbf{C}}=[\mathbf{c}_{1},\cdots,\mathbf{c}_{K_{1}}]$,
$\boldsymbol{\omega}=\left(  \omega_{1}{,}\cdots,\omega{_{K_{1}}}\right)
^{T}$, and
\begin{equation}
\omega_{k}=\phi(\hat{y}_{k})\phi^{\prime}(\hat{y}_{k})\mathbf{c}_{k}%
^{T}\left(  {\mathbf{C\boldsymbol{\hat{\Phi}}C}}^{T}\right)  ^{-1}%
\mathbf{c}_{k}\text{, }k=1,\cdots,K_{1}\text{.} \label{omegak}%
\end{equation}

In the following we will discuss how to get the optimal solution of
$\mathbf{C}$ for two specific cases.

\subsection{Algorithms for Optimization Objective Function}

\label{Algs}

\subsubsection{Algorithm 1: ${K_{0}}=K_{1}$}

\label{Alg.1}

Now ${\mathbf{CC}}^{T}={\mathbf{C}}^{T}{\mathbf{C}}=\mathbf{I}_{K_{1}}$, then
by Eq. (\ref{QC}) we have%
\begin{align}
&  Q_{1}{\left[  {\mathbf{C}}\right]  }=-\left\langle \sum_{k=1}^{K_{1}}%
\ln\left(  \phi(\hat{y}_{k})\right)  \right\rangle _{\mathbf{\hat{x}}}\text{,
}\label{Q0}\\
&  \frac{dQ_{1}{\left[  {\mathbf{C}}\right]  }}{d{\mathbf{C}}}=-\left\langle
\mathbf{\hat{x}}\boldsymbol{\omega}^{T}\right\rangle _{\mathbf{\hat{x}}%
}\text{,}\label{dQC}\\
&  \omega_{k}=\frac{\phi^{\prime}(\hat{y}_{k})}{\phi(\hat{y}_{k})}\text{,
}k=1,\cdots,K_{1}\text{.} \label{omg1}%
\end{align}
Under the orthogonality constraints (Eq. \ref{CC}), we can use the following
update rule for learning $\mathbf{C}$
\citep{Edelman(1998-geometry),Amari(1999-natural)}:%
\begin{align}
\mathbf{C}^{t+1}  &  =\mathbf{C}^{t}+\mu_{t}\frac{d\mathbf{C}^{t}}{dt}%
\text{,}\label{Ct}\\
\frac{d\mathbf{C}^{t}}{dt}  &  =-\frac{dQ_{1}{\left[  {\mathbf{C}}^{t}\right]
}}{d{\mathbf{C}}^{t}}+{\mathbf{C}}^{t}\left(  \frac{dQ_{1}{\left[
{\mathbf{C}}^{t}\right]  }}{d{\mathbf{C}}^{t}}\right)  ^{T}{\mathbf{C}}%
^{t}\text{,} \label{dC}%
\end{align}
where the learning rate parameter $\mu_{t}$ changes with the iteration count
$t$, $t=1,\cdots,t_{\max}$. Here we can use the empirical average to
approximate the integral in (\ref{dQC}) (see Eq. \ref{Ia_1}). We can also
apply stochastic gradient descent (SGD) method for online updating of
$\mathbf{C}^{t+1}$ in (\ref{Ct}).

The orthogonality constraint (Eq. \ref{CC}) can accelerate the convergence
rate. In practice, the orthogonal constraint (\ref{CC}) for objective function
(\ref{QC}) is not strictly necessary in this case. We can completely discard
this constraint condition and consider
\begin{equation}
\mathsf{minimize}\;{Q_{2}\left[  {\mathbf{C}}\right]  =}-\left\langle
\sum_{k=1}^{K_{1}}\ln\left(  \phi\left(  \hat{y}_{k}\right)  \right)
\right\rangle _{\mathbf{\hat{x}}}-\frac{1}{2}\ln\left(  \det\left(
{\mathbf{C}}^{T}\mathbf{C}\right)  \right)  \text{,} \label{QC1}%
\end{equation}
where we assume $\mathrm{rank}\left(  {\mathbf{C}}\right)  =K_{1}={K_{0}}$. If
we let
\begin{equation}
\frac{d\mathbf{C}}{dt}=-\mathbf{CC}^{T}\frac{d{Q_{2}\left[  {\mathbf{C}%
}\right]  }}{d{\mathbf{C}}}\text{,} \label{dCdt}%
\end{equation}
then%
\begin{equation}
\mathrm{Tr}\left(  \frac{d{Q_{2}\left[  {\mathbf{C}}\right]  }}{d{\mathbf{C}}%
}\frac{d\mathbf{C}^{T}}{dt}\right)  =-\mathrm{Tr}\left(  \mathbf{C}^{T}%
\frac{d{Q_{2}\left[  {\mathbf{C}}\right]  }}{d{\mathbf{C}}}\frac
{d{Q_{2}\left[  {\mathbf{C}}\right]  }}{d{\mathbf{C}}^{T}}\mathbf{C}\right)
\leq0\text{.} \label{trdQ}%
\end{equation}
Therefore we can use an update rule similar to Eq. \ref{Ct} for learning
$\mathbf{C}$. In fact, the method can also be extended to the case ${K_{0}%
}>K_{1}$ by using the same objective function (\ref{QC1}).

The learning rate parameter $\mu_{t}$ (see \ref{Ct}) is updated adaptively, as
follows. First, calculate
\begin{align}
\mu_{t}  &  =\frac{{v}_{t}}{\kappa_{t}}\text{, }t=1,\cdots,t_{\max}%
\text{,}\label{mut}\\
\kappa_{t}  &  =\frac{1}{K_{1}}\sum_{k=1}^{K_{1}}\frac{\left\Vert
\nabla\mathbf{C}^{t}(:,k)\right\Vert }{\left\Vert \mathbf{C}^{t}%
(:,k)\right\Vert }\text{,} \label{dt}%
\end{align}
and $\mathbf{C}^{t+1}$ by (\ref{Ct}) and (\ref{dC}), then calculate the value
$Q_{1}{\left[  {\mathbf{C}}^{t+1}\right]  }$. If $Q_{1}{\left[  {\mathbf{C}%
}^{t+1}\right]  }<Q_{1}{\left[  {\mathbf{C}}^{t}\right]  }$, then let
${v}_{t+1}\leftarrow{v}_{t}$, continue for the next iteration; otherwise, let
${v}_{t}\leftarrow\tau v_{t}$, $\mu_{t}\leftarrow{v}_{t}/\kappa_{t}$ and
recalculate $\mathbf{C}^{t+1}$ and $Q_{1}{\left[  {\mathbf{C}}^{t+1}\right]
}$. Here $0<{v}_{1}<1$ and $0<\tau<1$\ are set as constants. After getting
$\mathbf{C}^{t+1}$ for each update, we employ a Gram--Schmidt
orthonormalization process for matrix $\mathbf{C}^{t+1}$, where the
orthonormalization process can accelerate the convergence. However, we can
discard the Gram--Schmidt orthonormalization process after iterative $t_{0}$
($>1$) epochs for more accurate optimization solution $\mathbf{C}$. In this
case, the objective function is given by the Eq. (\ref{QC1}). We can also
further optimize parameter $b$ by gradient descent.

When $K_{0}=K_{1}$, the objective function ${Q_{2}\left[  {\mathbf{C}}\right]
}$ in Eq.~(\ref{QC1}) without constraint is the same as the objective function
of infomax ICA (IICA) \citep{Bell(1995-information),Bell(1997-independent)},
and as a consequence we should get the same optimal solution ${\mathbf{C}}$.
Hence, in this sense, the IICA may be regarded as a special case of our
method. Our method has a wider range of applications and can handle more
generic situations. Our model is derived by neural populations with a huge
number of neurons and it is not restricted to additive noise model. Moreover,
our method has a faster convergence rate during training than IICA (see
Section \ref{results}).

\subsubsection{Algorithm 2: ${K_{0}}\leq K_{1}$}

\label{Alg.2}

In this case, it is computationally expensive to update $\mathbf{C}$ by using
the gradient of $Q$ (see Eq. \ref{dQ}), since it needs to compute the inverse
matrix for every $\mathbf{\hat{x}}$. Here we provide an alternative method for
learning the optimal $\mathbf{C}$. First, we consider the following inequalities.

\begin{proposition}
\label{Proposition 2} The following inequations hold,%
\begin{align}
\frac{1}{2}\left\langle \ln\left(  \det\left(  {\mathbf{C\boldsymbol{\hat
{\Phi}}}}\mathbf{\mathbf{C}}^{T}\right)  \right)  \right\rangle _{\mathbf{\hat
{x}}}  &  \leq\frac{1}{2}\ln\left(  \det\left(  {\mathbf{C}}\left\langle
{\mathbf{\boldsymbol{\hat{\Phi}}}}\right\rangle _{\mathbf{\hat{x}}%
}\mathbf{\mathbf{C}}^{T}\right)  \right)  \text{,}\label{Cs0}\\
\left\langle \ln\left(  \det\left(  {\mathbf{C}}\boldsymbol{\Phi
}\mathbf{\mathbf{C}}^{T}\right)  \right)  \right\rangle _{\mathbf{\hat{x}}}
&  \leq\ln\left(  \det\left(  \mathbf{C}\left\langle \boldsymbol{\Phi
}\right\rangle _{\mathbf{\hat{x}}}\mathbf{\mathbf{C}}^{T}\right)  \right)
\label{Cs1}\\
&  \leq\frac{1}{2}\ln\left(  \det\left(  {\mathbf{C}}\left\langle
\boldsymbol{\Phi}\right\rangle _{\mathbf{\hat{x}}}^{2}\mathbf{\mathbf{C}}%
^{T}\right)  \right) \label{Cs1a}\\
&  \leq\frac{1}{2}\ln\left(  \det\left(  {\mathbf{C}}\left\langle
{\mathbf{\boldsymbol{\hat{\Phi}}}}\right\rangle _{\mathbf{\hat{x}}%
}\mathbf{\mathbf{C}}^{T}\right)  \right)  \text{,}\label{Cs1b}\\
\ln\left(  \det\left(  {\mathbf{C}}\boldsymbol{\Phi}\mathbf{\mathbf{C}}%
^{T}\right)  \right)   &  \leq\frac{1}{2}\ln\left(  \det\left(
{\mathbf{C\boldsymbol{\hat{\Phi}}}}\mathbf{\mathbf{C}}^{T}\right)  \right)
\text{,} \label{Cs2}%
\end{align}
where $\mathbf{C}\in%
\mathbb{R}
^{K_{0}\times K_{1}}$, $K_{0}\leq K_{1}$, and ${\mathbf{CC}}^{T}%
=\mathbf{I}_{K_{0}}$.
\end{proposition}

\begin{proof}
Functions $\ln\left(  \det\left(  {\mathbf{C}}\left\langle
{\mathbf{\boldsymbol{\hat{\Phi}}}}\right\rangle _{\mathbf{\hat{x}}%
}\mathbf{\mathbf{C}}^{T}\right)  \right)  $ and $\ln\left(  \det\left(
\mathbf{C}\left\langle \boldsymbol{\Phi}\right\rangle _{\mathbf{\hat{x}}%
}\mathbf{\mathbf{C}}^{T}\right)  \right)  $ are concave functions about
$p\left(  \mathbf{\hat{x}}\right)  $
\citep[see the proof of Proposition 5.2. in][]{Huang(2016-information)}, which
fact establishes inequalities (\ref{Cs0}) and (\ref{Cs1}).

Next we will prove the inequality (\ref{Cs2}). By SVD, we have
\begin{equation}
{\mathbf{C}}\boldsymbol{\Phi}=\mathbf{\ddot{U}\ddot{D}\ddot{V}}^{T}\text{,}
\label{CP}%
\end{equation}
where $\mathbf{\ddot{U}}$ is a $K_{0}\times K_{0}$ unitary orthogonal matrix,
$\mathbf{\ddot{V}}=[\mathbf{\ddot{v}}_{1},\mathbf{\ddot{v}}_{2},\cdots
,\mathbf{\ddot{v}}_{K_{1}}]\ $is an $K_{1}\times K_{1}$\ unitary orthogonal
matrix, and $\mathbf{\ddot{D}}$ is an $K_{0}\times K_{1}$ rectangular diagonal
matrix with $K_{0}$ positive real numbers on the diagonal. {By the matrix
Hadamard's inequality and Cauchy--Schwarz inequality we have }%
\begin{align}
&  \det\left(  {\mathbf{C}}\boldsymbol{\Phi}\mathbf{\mathbf{C}}^{T}%
{\mathbf{C}}\boldsymbol{\Phi}\mathbf{\mathbf{C}}^{T}\right)  \det\left(
{\mathbf{C\boldsymbol{\hat{\Phi}}C}}^{T}\right)  ^{-1}\nonumber\\
&  =\det\left(  \mathbf{\ddot{D}\ddot{V}}^{T}\mathbf{\mathbf{C}}%
^{T}{\mathbf{C}}\mathbf{\ddot{V}\ddot{D}}^{T}\left(  \mathbf{\ddot{D}\ddot{D}%
}^{T}\right)  ^{-1}\right) \nonumber\\
&  =\det\left(  \mathbf{\ddot{V}}_{1}^{T}\mathbf{\mathbf{C}}^{T}{\mathbf{C}%
}\mathbf{\ddot{V}}_{1}\right) \nonumber\\
&  =\det\left(  {\mathbf{C}}\mathbf{\ddot{V}}_{1}\right)  ^{2}\nonumber\\
&  \leq\prod_{k=1}^{K_{0}}\left(  {\mathbf{C}}\mathbf{\ddot{V}}_{1}\right)
_{k,k}^{2}\nonumber\\
&  \leq\prod_{k=1}^{K_{0}}\left(  {\mathbf{CC}}^{T}\right)  _{k,k}^{2}\left(
\mathbf{\ddot{V}}_{1}^{T}\mathbf{\ddot{V}}_{1}\right)  _{k,k}^{2}\nonumber\\
&  =1\text{,} \label{=1}%
\end{align}
where $\mathbf{\ddot{V}}_{1}=[\mathbf{\ddot{v}}_{1},\mathbf{\ddot{v}}%
_{2},\cdots,\mathbf{\ddot{v}}_{K_{0}}]\in%
\mathbb{R}
^{K_{1}\times K_{0}}$. The last equality holds because of ${\mathbf{CC}}%
^{T}=\mathbf{I}_{K_{0}}$ and $\mathbf{\ddot{V}}_{1}^{T}\mathbf{\ddot{V}}%
_{1}=\mathbf{I}_{K_{0}}$. This establishes inequality (\ref{Cs2}) and the
equality holds if and only if $K_{0}=K_{1}$ or$\ {\mathbf{C}}\mathbf{\ddot{V}%
}_{1}=\mathbf{I}_{K_{0}}$.

Similarly, we get inequality (\ref{Cs1a}):
\begin{equation}
\ln\left(  \det\left(  \mathbf{C}\left\langle \boldsymbol{\Phi}\right\rangle
_{\mathbf{\hat{x}}}\mathbf{\mathbf{C}}^{T}\right)  \right)  \leq\frac{1}{2}%
\ln\left(  \det\left(  {\mathbf{C}}\left\langle \boldsymbol{\Phi}\right\rangle
_{\mathbf{\hat{x}}}^{2}\mathbf{\mathbf{C}}^{T}\right)  \right)  \text{.}
\label{pcs1a}%
\end{equation}
By Jensen's inequality, we have
\begin{equation}
\left\langle \phi\left(  \hat{y}_{k}\right)  \right\rangle _{\mathbf{\hat{x}}%
}^{2}\leq\left\langle \phi\left(  \hat{y}_{k}\right)  ^{2}\right\rangle
_{\mathbf{\hat{x}}}\text{, }\forall k=1,\cdots,K_{1}\text{.} \label{pcs1b}%
\end{equation}
Then it follows from (\ref{pcs1b})\ that inequality (\ref{Cs1b}) holds:%
\begin{equation}
\frac{1}{2}\ln\left(  \det\left(  {\mathbf{C}}\left\langle \boldsymbol{\Phi
}\right\rangle _{\mathbf{\hat{x}}}^{2}\mathbf{\mathbf{C}}^{T}\right)  \right)
\leq\frac{1}{2}\ln\left(  \det\left(  {\mathbf{C}}\left\langle
{\mathbf{\boldsymbol{\hat{\Phi}}}}\right\rangle _{\mathbf{\hat{x}}%
}\mathbf{\mathbf{C}}^{T}\right)  \right)  \text{.} \label{pcs1c}%
\end{equation}
This completes the proof of \textbf{Proposition \ref{Proposition 2}}.
\end{proof}

By \textbf{Proposition \ref{Proposition 2}, }if $K_{0}=K_{1}$ then we get%
\begin{align}
\frac{1}{2}\left\langle \ln\left(  \det\left(  {\mathbf{\boldsymbol{\hat{\Phi
}}}}\right)  \right)  \right\rangle _{\mathbf{\hat{x}}}  &  \leq\frac{1}{2}%
\ln\left(  \det\left(  \left\langle {\mathbf{\boldsymbol{\hat{\Phi}}}%
}\right\rangle _{\mathbf{\hat{x}}}\right)  \right)  \text{,}\label{case2.1}\\
\left\langle \ln\left(  \det\left(  \boldsymbol{\Phi}\right)  \right)
\right\rangle _{\mathbf{\hat{x}}}  &  \leq\ln\left(  \det\left(  \left\langle
\boldsymbol{\Phi}\right\rangle _{\mathbf{\hat{x}}}\right)  \right)
\label{case2.2}\\
&  =\frac{1}{2}\ln\left(  \det\left(  \left\langle \boldsymbol{\Phi
}\right\rangle _{\mathbf{\hat{x}}}^{2}\right)  \right) \label{case2.3}\\
&  \leq\frac{1}{2}\ln\left(  \det\left(  \left\langle
{\mathbf{\boldsymbol{\hat{\Phi}}}}\right\rangle _{\mathbf{\hat{x}}}\right)
\right)  \text{,}\label{case2.4}\\
\ln\left(  \det\left(  \boldsymbol{\Phi}\right)  \right)   &  =\frac{1}{2}%
\ln\left(  \det\left(  {\mathbf{\boldsymbol{\hat{\Phi}}}}\right)  \right)
\text{.} \label{case2.5}%
\end{align}
On the other hand, it follows from (\ref{Q1}) and \textbf{Proposition
\ref{Proposition 2}} that
\begin{align}
\left\langle \ln\left(  \det\left(  {\mathbf{C}}\boldsymbol{\Phi
}\mathbf{\mathbf{C}}^{T}\right)  \right)  \right\rangle _{\mathbf{\hat{x}}}
&  \leq-Q\leq\frac{1}{2}\ln\left(  \det\left(  {\mathbf{C}}\left\langle
{\mathbf{\boldsymbol{\hat{\Phi}}}}\right\rangle _{\mathbf{\hat{x}}%
}\mathbf{\mathbf{C}}^{T}\right)  \right)  \text{,}\label{case2.6}\\
\left\langle \ln\left(  \det\left(  {\mathbf{C}}\boldsymbol{\Phi
}\mathbf{\mathbf{C}}^{T}\right)  \right)  \right\rangle _{\mathbf{\hat{x}}}
&  \leq-\hat{Q}\leq\frac{1}{2}\ln\left(  \det\left(  {\mathbf{C}}\left\langle
{\mathbf{\boldsymbol{\hat{\Phi}}}}\right\rangle _{\mathbf{\hat{x}}%
}\mathbf{\mathbf{C}}^{T}\right)  \right)  \text{.} \label{case2.7}%
\end{align}
Hence we can see that $\hat{Q}$ is close to $Q\ $(see \ref{Q1}). Moreover, it
follows from the Cauchy--Schwarz inequality that%
\begin{equation}
\left\langle \left(  \boldsymbol{\Phi}\right)  _{k,k}\right\rangle
_{{\mathbf{\hat{x}}}}=\left\langle \phi\left(  \hat{y}_{k}\right)
\right\rangle _{\hat{y}_{k}}\leq\left(  \int\phi\left(  \hat{y}_{k}\right)
^{2}d\hat{y}_{k}\int p\left(  \hat{y}_{k}\right)  ^{2}d\hat{y}_{k}\right)
^{1/2}\text{,} \label{gk'<}%
\end{equation}
where $k=1,\cdots,K_{1}$, the equality holds if and only if the following
holds:%
\begin{equation}
p\left(  \hat{y}_{k}\right)  =\frac{\phi\left(  \hat{y}_{k}\right)  }{\int%
\phi\left(  \hat{y}_{k}\right)  d\hat{y}_{k}}\text{, }k=1,\cdots,K_{1}\text{,}
\label{pyk}%
\end{equation}
which is the similar to Eq. (\ref{py}).

Since $I(X;R)=I(Y;R)$ (see \textbf{Proposition }\ref{Proposition 1}), by
maximizing $I(X;R)$ we hope the equality in inequality (\ref{3.6}) and
equality (\ref{py}) hold, at least approximatively. On the other hand, let
\begin{align}
{\mathbf{C}}^{opt}  &  =\arg\min_{\mathbf{C}}Q{\left[  {\mathbf{C}}\right]
}=\arg\max_{\mathbf{C}}\left(  \left\langle \ln\left(  \det(\mathbf{C}%
\boldsymbol{\hat{\Phi}}{\mathbf{C}}^{T})\right)  \right\rangle _{\mathbf{\hat
{x}}}\right)  \text{,}\label{Copt^}\\
{\mathbf{\hat{C}}}^{opt}  &  =\arg\min_{\mathbf{C}}\hat{Q}{\left[
{\mathbf{C}}\right]  }=\arg\max_{\mathbf{C}}\left(  \ln\left(  \det\left(
{\mathbf{C}}\left\langle \boldsymbol{\Phi}\right\rangle _{\mathbf{\hat{x}}%
}^{2}{\mathbf{C}}^{T}\right)  \right)  \right)  \text{,} \label{Copt}%
\end{align}
${\mathbf{C}}^{opt}$ and ${\mathbf{\hat{C}}}^{opt}$ make (\ref{py}) and
(\ref{pyk}) to hold true, which implies that they are the same optimal
solution: ${\mathbf{C}}^{opt}={\mathbf{\hat{C}}}^{opt}$.

Therefore, we can use the following objective function $\hat{Q}{\left[
{\mathbf{C}}\right]  }$ as a substitute for $Q{\left[  {\mathbf{C}}\right]  }$
and write the optimization problem as:
\begin{align}
&  \mathsf{minimize}\;\;\hat{Q}{\left[  {\mathbf{C}}\right]  =}-\frac{1}{2}%
\ln\left(  \det\left(  {\mathbf{C}}\left\langle \boldsymbol{\Phi}\right\rangle
_{\mathbf{\hat{x}}}^{2}{\mathbf{C}}^{T}\right)  \right)  \text{,}%
\label{obj2}\\
&  \mathsf{subject\;to}\;{\mathbf{CC}}^{T}=\mathbf{I}_{K_{0}}\text{.}
\label{cons2}%
\end{align}
The update rule (\ref{Ct}) may also apply here and a modified algorithm
similar to Algorithm 1 may be used for parameter learning.

\subsection{Supplementary Experiments}

\label{SupExp}

\subsubsection{Quantitative Methods for Comparison}

\label{quant}

To quantify the efficiency of learning representations by the above
algorithms, we calculate the coefficient entropy (CFE) for estimating coding
cost as follows \citep{Lewicki(1999-probabilistic),Lewicki(2000-learning)}:
\begin{align}
\check{y}_{k}  &  =\zeta\mathbf{\check{w}}_{k}^{T}\mathbf{\check{x}}\text{,
}k=1,\cdots,K_{1}\text{,}\label{yk1}\\
\zeta &  =\frac{K_{1}}{\sum_{k=1}^{K_{1}}\left\Vert \mathbf{\check{w}}%
_{k}\right\Vert }\text{,} \label{yk1.a}%
\end{align}
where $\mathbf{\check{x}}$ is defined by Eq. (\ref{xv}), and $\mathbf{\check
{w}}_{k}$ is the corresponding optimal filter. To estimate the probability
density of coefficients $q_{k}(\check{y}_{k})$ ($k=1,\cdots,K_{1}$) from the
$M$ training samples, we apply the kernel density estimation for $q_{k}%
(\check{y}_{k})$ and use a normal kernel with an adaptive optimal window
width. Then we define the CFE $h$ as%
\begin{align}
&  h=\frac{1}{K_{1}}\sum_{k=1}^{K_{1}}H_{k}(\check{Y}_{k})\text{,}\label{h}\\
&  H_{k}(\check{Y}_{k})=-\Delta%
{\textstyle\sum_{n}}
q_{k}(n\Delta)\log_{2}q_{k}(n\Delta)\text{,} \label{HYk}%
\end{align}
where $q_{k}(\check{y}_{k})$ is quantized as discrete $q_{k}(n\Delta)$ and
$\Delta$\ is the step size.

Methods such as IICA and SRBM as well as our methods have feedforward
structures in which information is transferred directly through a nonlinear
function, e.g., the sigmoid function. We can use the amount of transmitted
information to measure the results learned by these methods. Consider a neural
population with $N$ neurons, which is a stochastic system with nonlinear
transfer functions. We chose a sigmoidal transfer function and Gaussian noise
with standard deviation set to $1$ as the system noise. In this case, from
(\ref{Ia}), (\ref{PoissNeuron.2}) and (\ref{GaussNeuron.2}), we see that the
approximate MI $I_{G}$ is equivalent to the case of the Poisson neuron model.
It follows from (\ref{Ia3})--(\ref{eps1}) that
\begin{align}
&  I(X;R)=I\left(  \tilde{X};R\right)  =H(\tilde{X})-H\left(  \tilde
{X}|R\right)  \simeq\tilde{I}_{G}=H(\tilde{X})-h_{1}\text{,}\label{IXR}\\
&  H\left(  \tilde{X}|R\right)  \simeq h_{1}=-\frac{1}{2}\left\langle
\ln\left(  \det\left(  \frac{1}{2\pi e}\left(  NK_{0}^{-1}%
\mathbf{C\boldsymbol{\hat{\Phi}}C}^{T}+\mathbf{I}_{K_{0}}\right)  \right)
\right)  \right\rangle _{\mathbf{\hat{x}}}\text{,} \label{h1}%
\end{align}
where we set $N=10^{6}$. A good representation should make the MI $I(X;R)$ as
big as possible. Equivalently, for the same inputs, a good representation
should make the conditional entropy (CDE) $H\left(  \tilde{X}|R\right)  $ (or
$h_{1}$) as small as possible.

\subfiglabelskip=0pt \begin{figure}[tbh]
\vskip -0.2in \centering
\subfigure[]{\label{Fig3a}
\includegraphics[width= .315\linewidth]{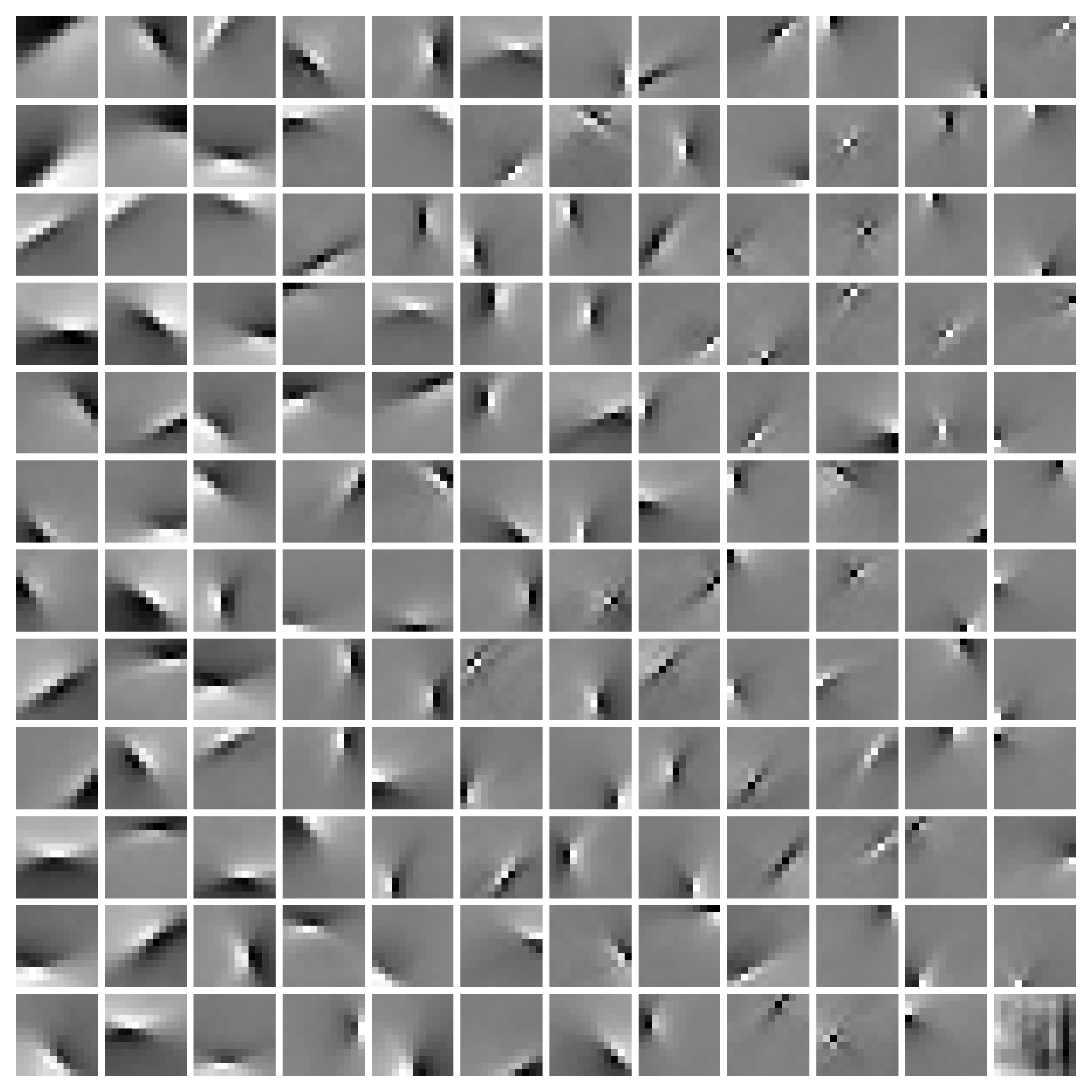}} \hspace{0pt}
\subfigure[]{\label{Fig3b}
\includegraphics[width= .315\linewidth]{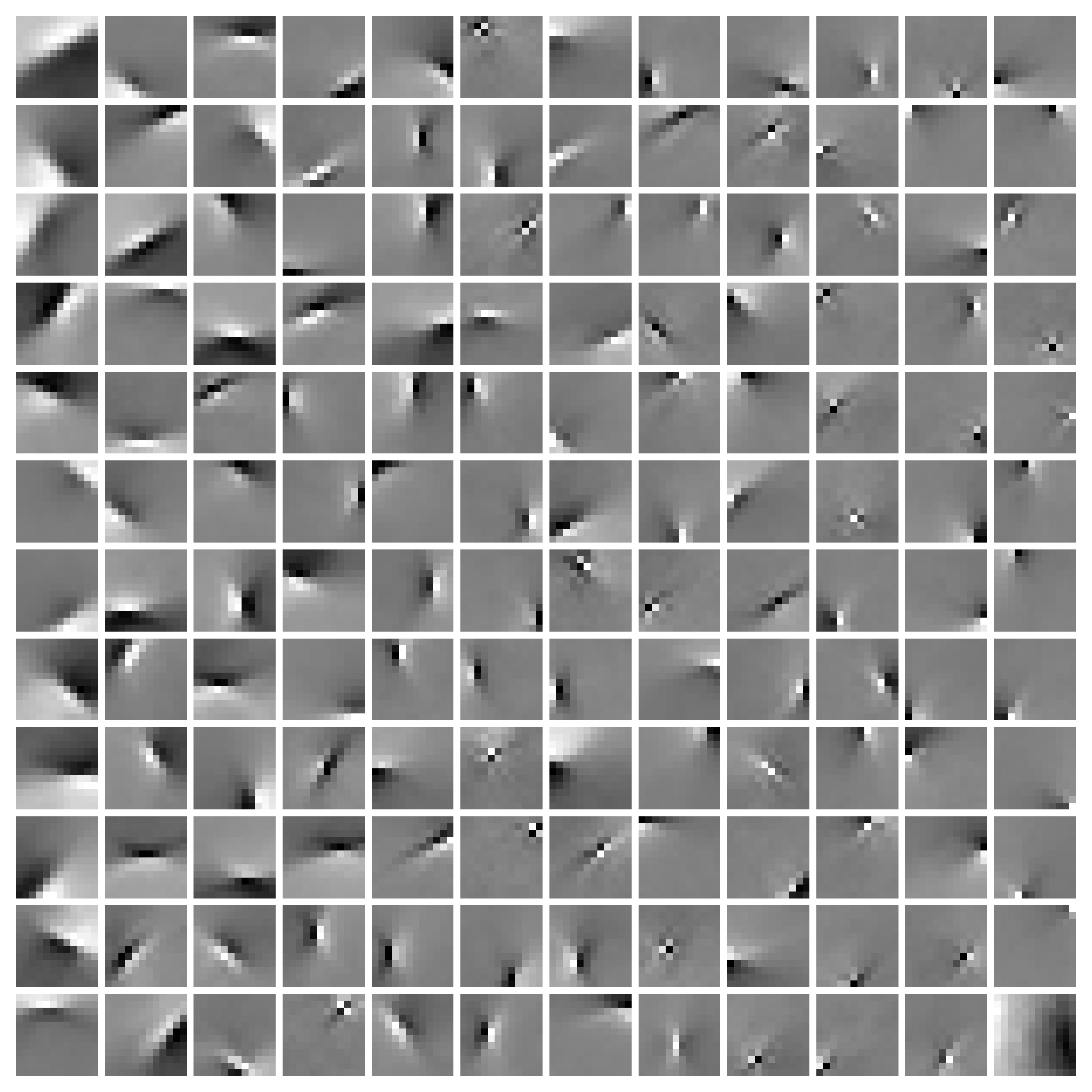}} \hspace{0pt}
\subfigure[]{\label{Fig3c}
\includegraphics[width= .315\linewidth]{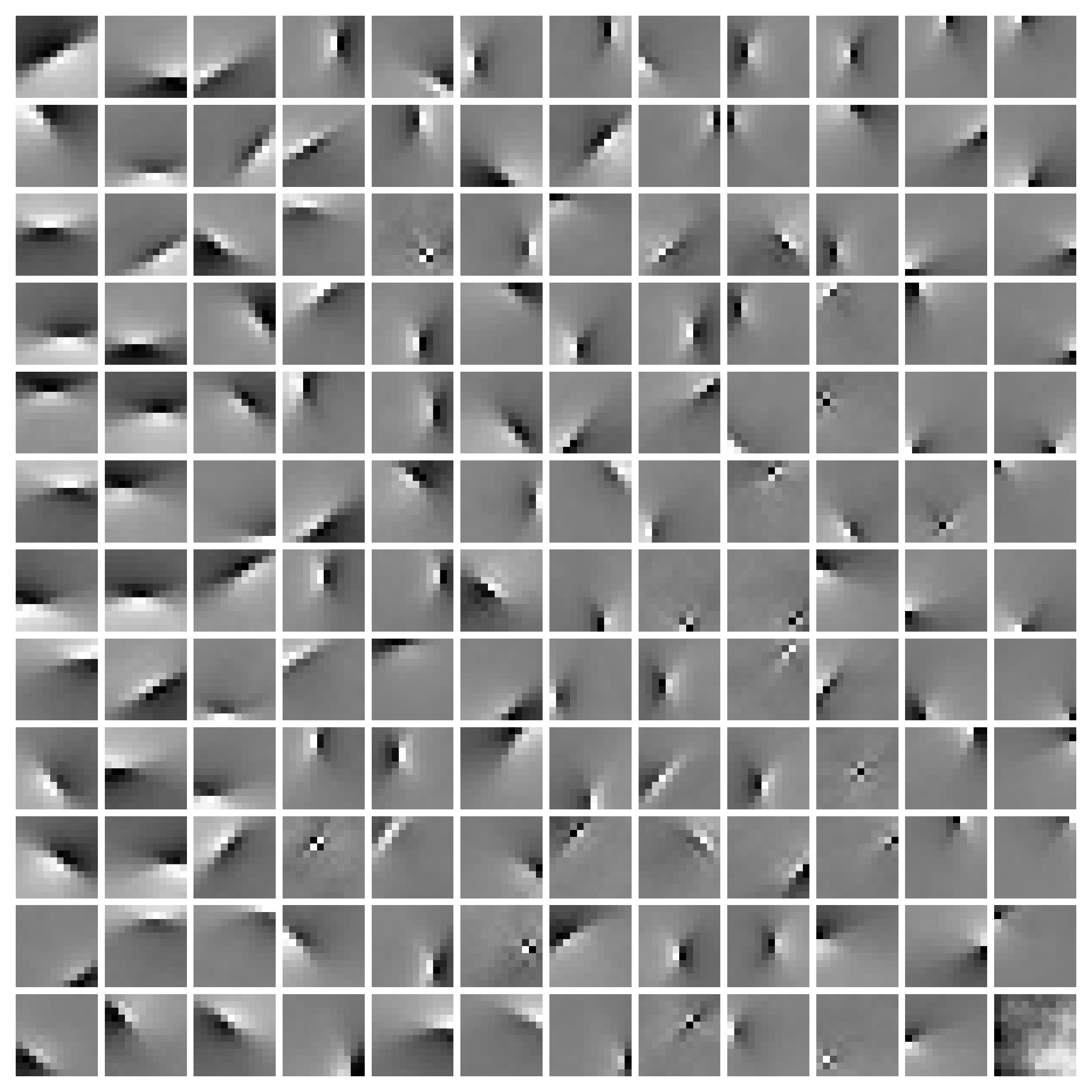}} \hspace{0pt}
\subfigure[]{\label{Fig3d}
\includegraphics[width= .315\linewidth]{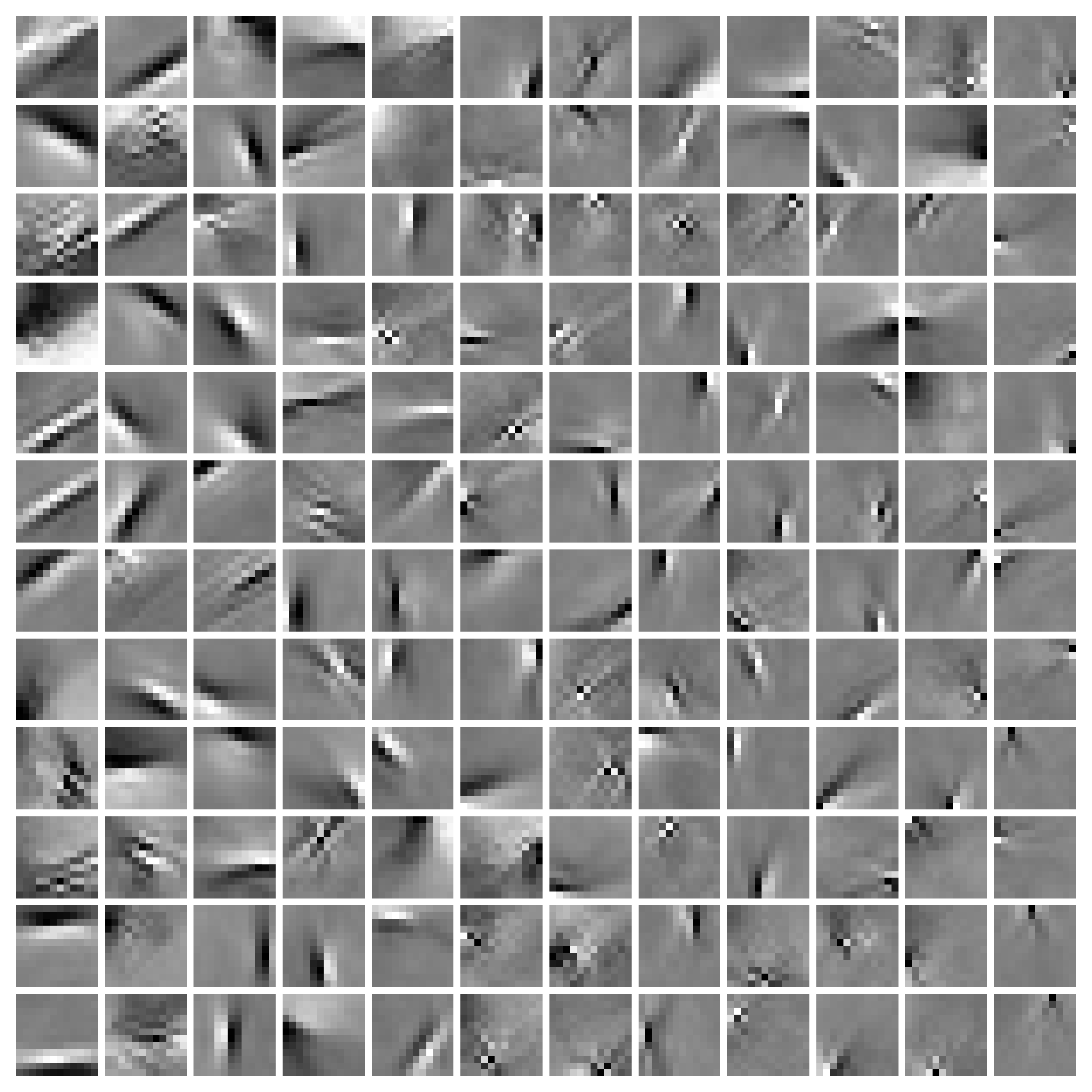}} \hspace{0pt}
\subfigure[]{\label{Fig3e}
\includegraphics[width= .315\linewidth]{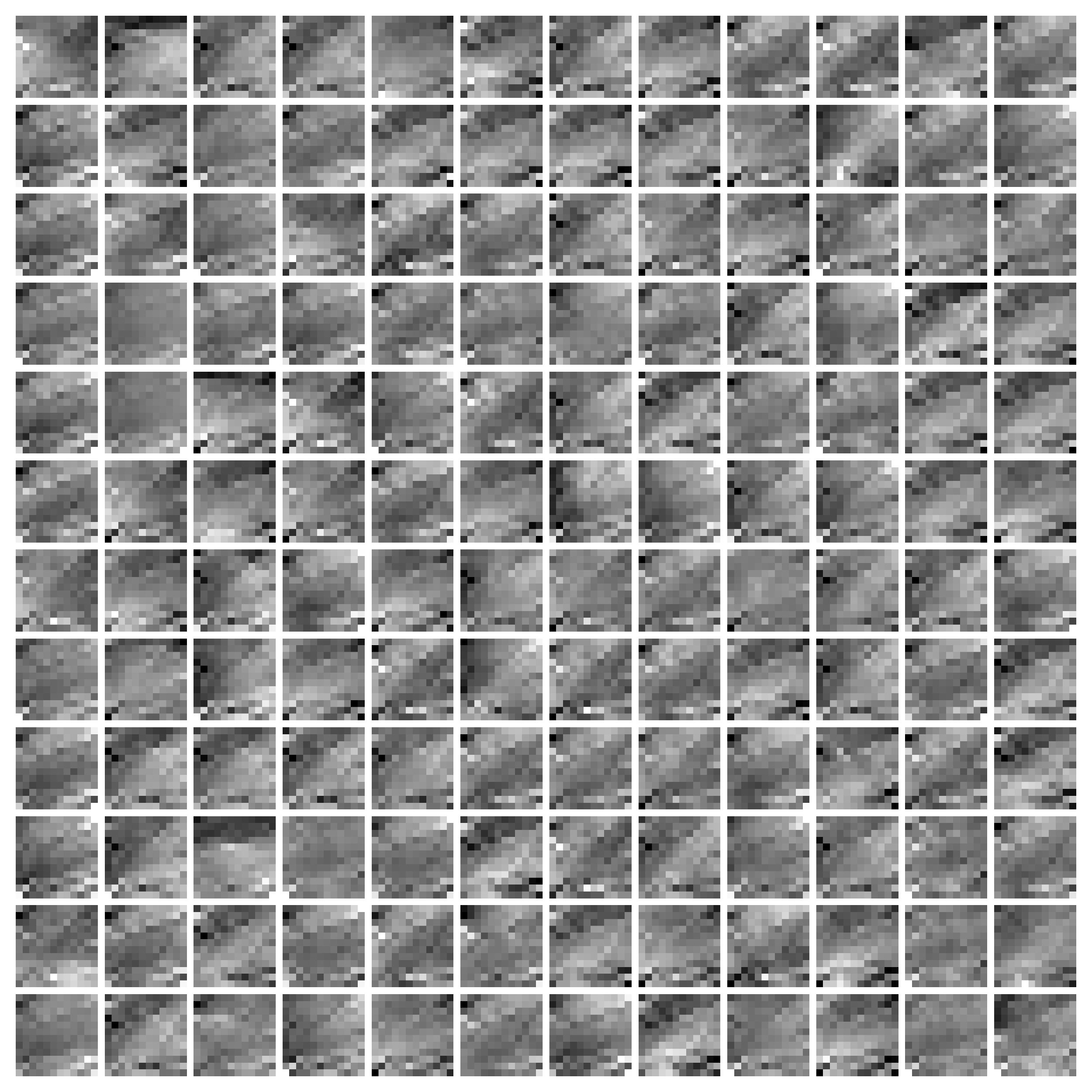}} \hspace{0pt}
\subfigure[]{\label{Fig3f}
\includegraphics[width= .315\linewidth]{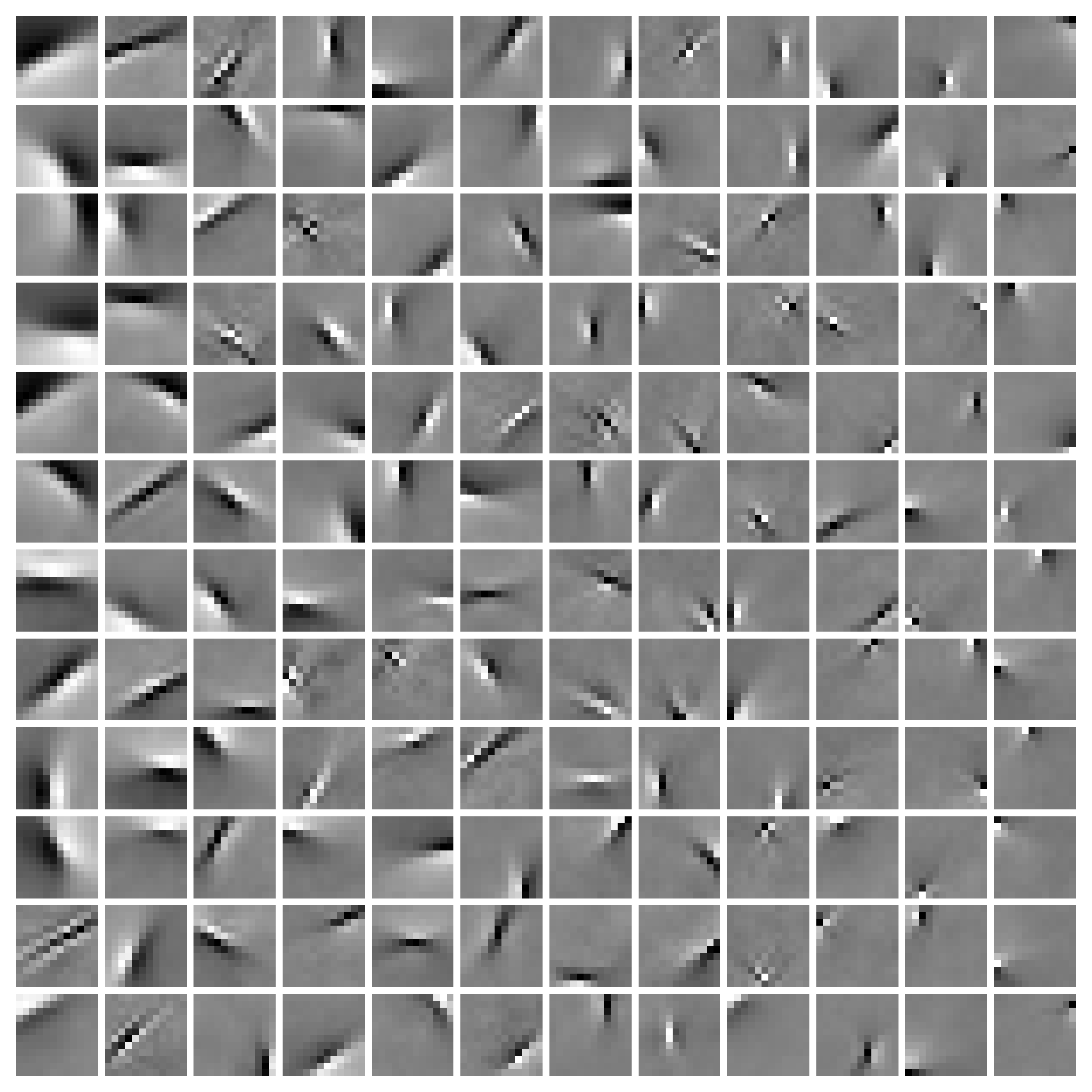}} \hspace{0pt}
\vskip -0.15 in\caption{Comparison of basis vectors obtained by our method and
other methods. Panel (\textbf{a})--(\textbf{e}) correspond to panel
(\textbf{a})--(\textbf{e}) in Figure~\ref{Fig1}, where the basis vectors are
given by (\ref{B_1}). The basis vectors in panel (\textbf{f}) are learned by
MBDL and given by (\ref{B_2}).}%
\label{Fig3}%
\end{figure}

\subsubsection{Comparison of Basis Vectors}

We compared our algorithm with an up-to-date sparse coding algorithm, the
mini-batch dictionary learning (MBDL) as given in
\citep{Mairal(2009-IP-online),Mairal(2010-online)} and integrated in Python
library, i.e. scikit-learn. The input data was the same as the above, i.e.
$10^{5}$ nature image patches preprocessed by the ZCA whitening filters.

We denotes the optimal dictionary learned by MBDL as $\mathbf{\check{B}}\in%
\mathbb{R}
^{K\times K_{1}}$ for which each column represents a basis vector. Now we
have
\begin{align}
\mathbf{x}  &  \approx\mathbf{{\mathbf{U}\boldsymbol{\Sigma}}}^{1/2}%
\mathbf{{\mathbf{U}}}^{T}\mathbf{\check{B}y}=\mathbf{\tilde{B}y}%
\text{,}\label{x.2}\\
\mathbf{\tilde{B}}  &  =\mathbf{{\mathbf{U}\boldsymbol{\Sigma}}}%
^{1/2}\mathbf{{\mathbf{U}}}^{T}\mathbf{\check{B}}\text{,} \label{B_2}%
\end{align}
where $\mathbf{y}=\left(  y_{1},\cdots,y_{K_{1}}\right)  ^{T}$ is the
coefficient vector.

Similarly, we can obtain a dictionary from the filter matrix $\mathbf{C}$.
Suppose $\mathrm{rank}\left(  {\mathbf{C}}\right)  ={K_{0}}\leq K_{1}$, then
it follows from (\ref{ya}) that
\begin{equation}
\mathbf{\hat{x}}=\left(  a{\mathbf{CC}}^{T}\right)  ^{-1}{\mathbf{C}%
}\mathbf{y}\text{.} \label{x^1}%
\end{equation}
By (\ref{x^}) and (\ref{x^1}), we get
\begin{align}
\mathbf{x}  &  \approx\mathbf{By}=a\mathbf{B}{\mathbf{C}}^{T}%
\mathbf{{\boldsymbol{\Sigma}}}_{0}^{-1/2}\mathbf{{\mathbf{U}}}_{0}%
^{T}\mathbf{x}\text{,}\label{x_1}\\
\mathbf{B}  &  =a^{-1}\mathbf{{\mathbf{U}}}_{0}\mathbf{{\boldsymbol{\Sigma}}%
}_{0}^{1/2}\left(  {\mathbf{CC}}^{T}\right)  ^{-1}{\mathbf{C}}=\left[
\mathbf{b}_{1},\cdots,\mathbf{b}_{K_{1}}\right]  \text{,} \label{B_1}%
\end{align}
where $\mathbf{y}=\mathbf{W}^{T}\mathbf{x}=a{\mathbf{C}}^{T}%
\mathbf{{\boldsymbol{\Sigma}}}_{0}^{-1/2}\mathbf{{\mathbf{U}}}_{0}%
^{T}\mathbf{x}$, the vectors $\mathbf{b}_{1},\cdots,\mathbf{b}_{K_{1}}$\ can
be regarded as the basis vectors and the strict equality holds when
$K_{0}=K_{1}=K$. Recall that $\mathbf{X}=[\mathbf{x}_{1}$\textrm{,\thinspace
}$\cdots$\textrm{,\thinspace}$\mathbf{x}_{M}]=\mathbf{US\tilde{V}}^{T}$ (see
Eq. \ref{Xsvd}) and $\mathbf{Y}=[\mathbf{y}_{1}$\textrm{,\thinspace\thinspace
}$\cdots$\textrm{,\thinspace}$\mathbf{y}_{M}]=\mathbf{W}^{T}\mathbf{X}%
=a\sqrt{M-1}{\mathbf{C}}^{T}\mathbf{\tilde{V}}_{0}^{T}$, then we get
$\mathbf{\breve{X}}=\mathbf{BY}=\sqrt{M-1}\mathbf{{\mathbf{U}}}_{0}%
\mathbf{{\boldsymbol{\Sigma}}}_{0}^{1/2}\mathbf{\tilde{V}}_{0}^{T}%
\approx\mathbf{X}$. Hence, Eq. (\ref{x_1}) holds.

The basis vectors shown in Figure~\ref{Fig3a}--\ref{Fig3e} correspond to
filters in Figure~\ref{Fig1a}--\ref{Fig1e}. And Figure~\ref{Fig3f} illustrates
the optimal dictionary $\mathbf{\tilde{B}}$ learned by MBDL, where we set the
regularization parameter as $\lambda=1.2/\sqrt{K}$, the batch size as $50$ and
the total number of iterations to perform as $20000$, which took about $3$
hours for training. From Figure~\ref{Fig3} we see that these basis vectors
obtained by the above algorithms have local Gabor-like shapes except for those
by SRBM. If $\mathrm{rank}(\mathbf{\check{B})}=K=K_{1}$, then the matrix
$\mathbf{\check{B}}^{-T}$\ can be regarded as a filter matrix like matrix
${\mathbf{\check{C}}}$ (see Eq. \ref{Cv}). However, from the column vector of
matrix $\mathbf{\check{B}}^{-T}$ we cannot find any local Gabor-like filter
that resembles the filters shown in Figure~\ref{Fig1}. Our algorithm has less
computational cost and a much faster convergence rate than the sparse coding
algorithm. Moreover, the sparse coding method involves a dynamic generative
model that requires relaxation and is therefore unsuitable for fast inference,
whereas the feedforward framework of our model is easy for inference because
it only requires evaluating the nonlinear tuning functions.

\subsubsection{Learning Overcomplete Bases}

\subfiglabelskip=0pt \begin{figure}[tbh]
\vskip -0.2in \centering
\subfigure[]{\label{Fig4a}
\includegraphics[width= .48\linewidth]{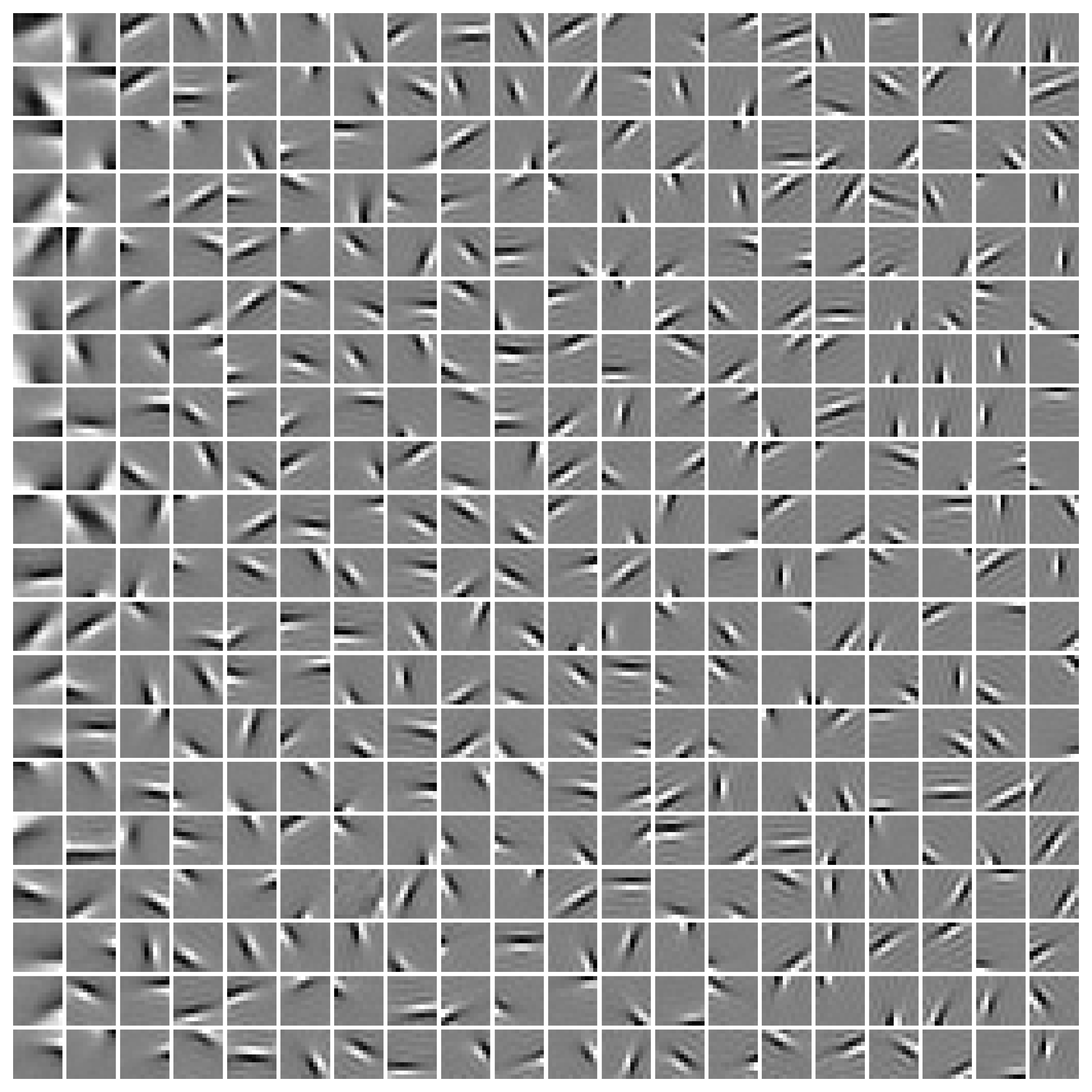}} \hspace{0pt}
\subfigure[]{\label{Fig4b}
\includegraphics[width= .48\linewidth]{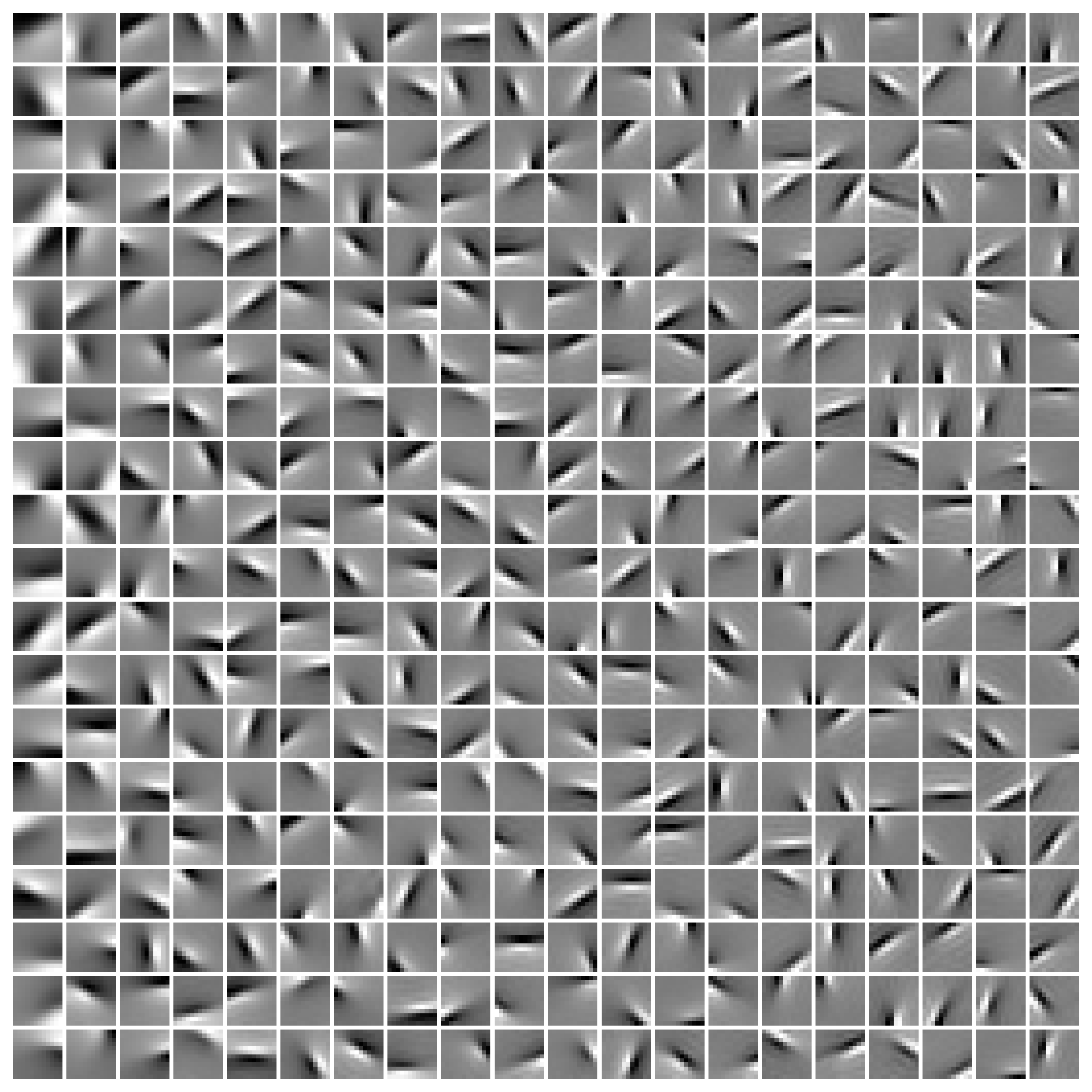}} \hspace{0pt}
\subfigure[]{\label{Fig4c}
\includegraphics[width= .48\linewidth]{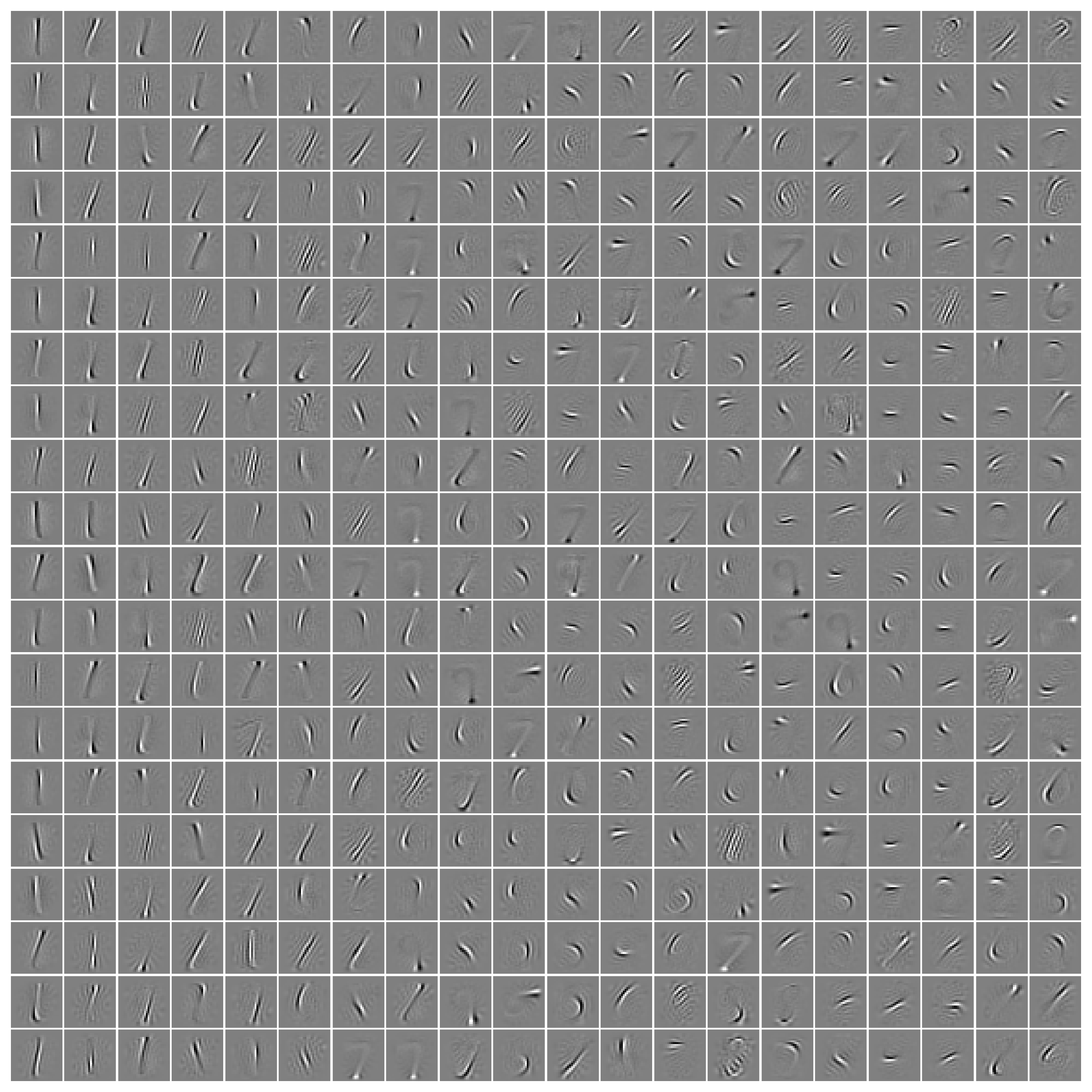}} \hspace{0pt}
\subfigure[]{\label{Fig4d}
\includegraphics[width= .48\linewidth]{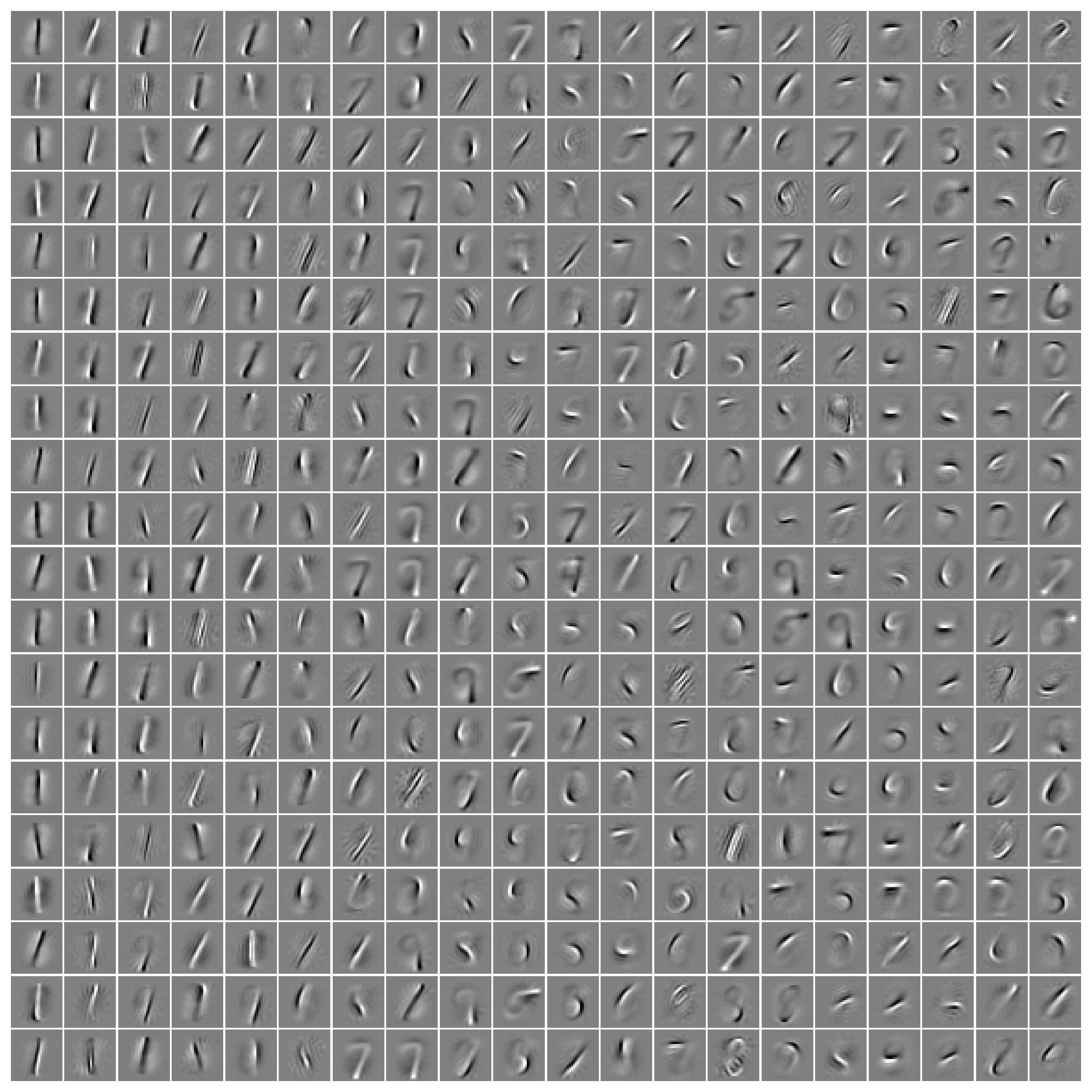}} \hspace{0pt}
\vskip -0.15in\caption{Filters and bases obtained from Olshausen's image
dataset and MNIST dataset by Algorithm 2. (\textbf{a}) and (\textbf{b}):
$400$\ typical filters and the corresponding bases obtained from Olshausen's
image dataset, where $K_{0}=82$ and $K_{1}=1024$. (\textbf{c}) and
(\textbf{d}): $400$\ typical filters and the corresponding bases obtained from
the MNIST dataset, where $K_{0}=183$ and $K_{1}=1024$. }%
\label{Fig4}%
\end{figure}

We have trained our model on the Olshausen's nature image patches with a
highly overcomplete setup by optimizing the objective (\ref{obj2}) by Alg.2
and got Gabor-like filters. The results of $400$\ typical filters chosen from
$1024$ output filters are displayed in Figure~\ref{Fig4a} and corresponding
base (see Eq. \ref{B_1}) are shown in Figure~\ref{Fig4b}. Here the parameters
are $K_{1}=1024$, $t_{\max}=100$, ${v}_{1}=0.4$, $\tau=0.8$, and
$\epsilon=0.98$ (see \ref{K0}), from which we got $\mathrm{rank}\left(
\mathbf{B}\right)  =K_{0}=82$. Compared to the ICA-like results in
Figure~\ref{Fig1a}--\ref{Fig1c}, the average size of Gabor-like filters in
Figure~\ref{Fig4a} is bigger, indicating that the small noise-like local
structures in the images have been filtered out.

We have also trained our model on 60,000 images of handwritten digits from
MNIST dataset \citep{LeCun(1998-gradient)} and the resultant $400$\ typical
optimal filters and bases are shown in Figure~\ref{Fig4c} and
Figure~\ref{Fig4d}, respectively. All parameters were the same as
Figure~\ref{Fig4a} and Figure~\ref{Fig4b}: $K_{1}=1024$, $t_{\max}=100$,
$v_{1}=0.4$, $\tau=0.8$ and $\epsilon=0.98$, from which we got $\mathrm{rank}%
\left(  \mathbf{B}\right)  =K_{0}=183$. From these figures we can see that the
salient features of the input images are reflected in these filters and bases.
We could also get the similar overcomplete filters and bases by SRBM and MBDL.
However, the results depended sensitively on the choice of parameters and the
training took a long time.

\subfiglabelskip=0pt \begin{figure}[t]
\vskip -0.2in \centering
\subfigure[]{\label{Fig5a}
\includegraphics[width= .48\linewidth]{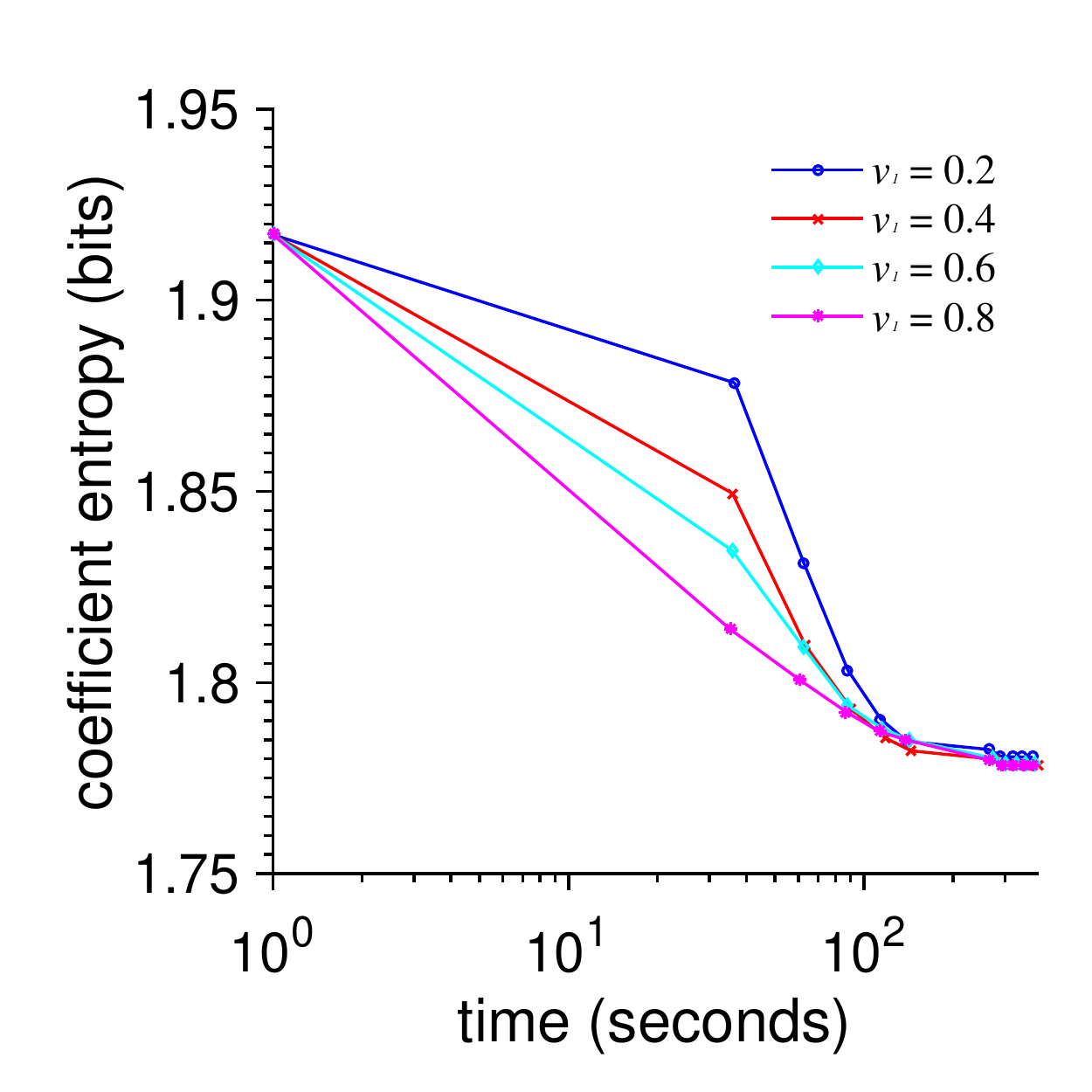}} \hspace{0pt}
\subfigure[]{\label{Fig5b}
\includegraphics[width= .48\linewidth]{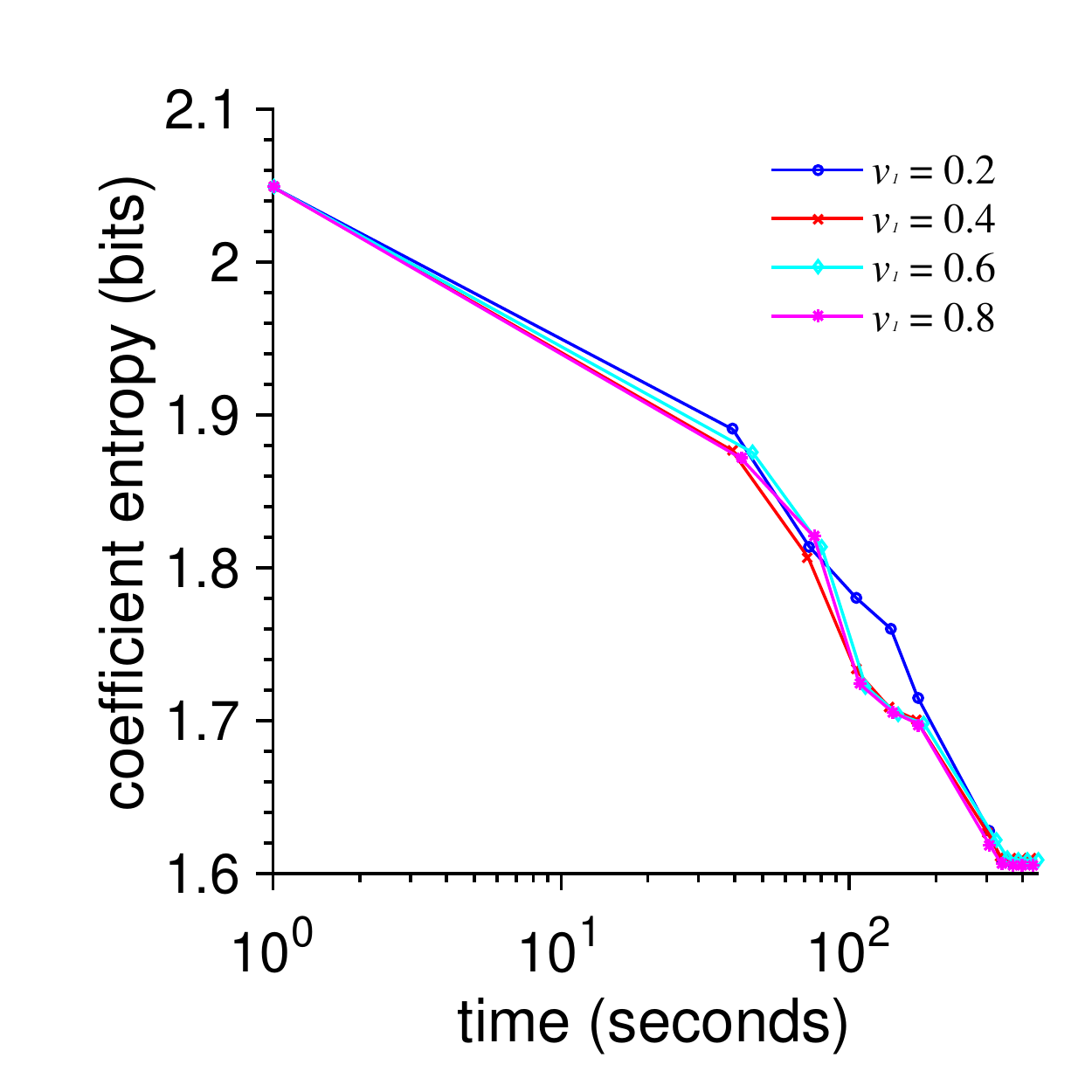}} \hspace{0pt}
\vskip -0.15in\caption{CFE as a function of training time for Alg.2, with
${v}_{1}=0.2$, $0.4$, $0.6$ or $0.8$. In all experiments parameters were set
to $t_{\max}=100$, $t_{0}=50$ and $\tau=0.8$. (\textbf{a}): corresponding to
Figure~\ref{Fig4a} or Figure~\ref{Fig4b}. (\textbf{b}): corresponding to
Figure~\ref{Fig4c} or Figure~\ref{Fig4d}.}%
\label{Fig5}%
\end{figure}

Figure~\ref{Fig5} shows that CFE as a function of training time for Alg.2,
where Figure~\ref{Fig5a} corresponds to Figure~\ref{Fig4a}-\ref{Fig4b} for
learning nature image patches and Figure~\ref{Fig5b} corresponds to
Figure~\ref{Fig4c}-\ref{Fig4d} for learning MNIST dataset. We set parameters
$t_{\max}=100$ and $\tau=0.8$ for all experiments and varied parameter
${v}_{1}$ for each experiment, with ${v}_{1}=0.2$, $0.4$, $0.6$ or $0.8$.
These results indicate a fast convergence rate for training on different
datasets. Generally, the convergence is insensitive to the change of parameter
${v}_{1}$.

We have also performed additional tests on other image datasets and got
similar results, confirming the speed and robustness of our learning method.
Compared with other methods, e.g., IICA, FICA, MBDL, SRBM or sparse
autoencoders etc., our method appeared to be more efficient and robust for
unsupervised learning of representations. We also found that complete and
overovercomplete filters and bases learned by our methods had local Gabor-like
shapes while the results by SRBM or MBDL did not have this property.

\subsubsection{Image Denoising}

Similar to the sparse coding method applied to image denoising
\citep{Elad(2006-image)}, our method (see Eq. \ref{B_1}) can also be applied
to image denoising, as shown by an example in Figure~\ref{Fig6}. The filters
or bases were learned by using $7\times7$ image patches sampled from the left
half of the image, and subsequently used to reconstruct the right half of the
image which was distorted by Gaussian noise. A common practice for evaluating
the results of image denoising is by looking at the difference between the
reconstruction and the original image. If the reconstruction is perfect the
difference should look like Gaussian noise. In Figure~\ref{Fig6c} and
\ref{Fig6d} a dictionary ($100$ bases) was learned by MBDL and orthogonal
matching pursuit was used to estimate the sparse solution. \footnote{Python
source code is available at
http://scikit-learn.org/stable/\_downloads/plot\_image\_denoising.py} For our
method (shown in Figure~\ref{Fig6b}), we first get the optimal filters
parameter $\mathbf{W}$, a low rank matrix ($K_{0}<K$), then from the distorted
image patches $\mathbf{x}_{m}$ ($m=1,\cdots,M$) we get filter outputs
$\mathbf{y}_{m}=\mathbf{W}^{T}\mathbf{x}_{m}$ and the reconstruction
$\mathbf{\breve{x}}_{m}=\mathbf{By}_{m}$ (parameters: $\epsilon=0.975$ and
$K_{0}=K_{1}=14$). As can be seen from Figure~\ref{Fig6}, our method worked
better than dictionary learning, although we only used $14$ bases compared
with $100$ bases used by dictionary learning. Our method is also more
efficient. We can get better optimal bases $\mathbf{B}$ by a generative model
using our infomax approach (details not shown).

\subfiglabelskip=0pt \begin{figure}[t]
\vskip -0.0in \centering
\subfigure[]{\label{Fig6a}
\includegraphics[width= .48\linewidth]{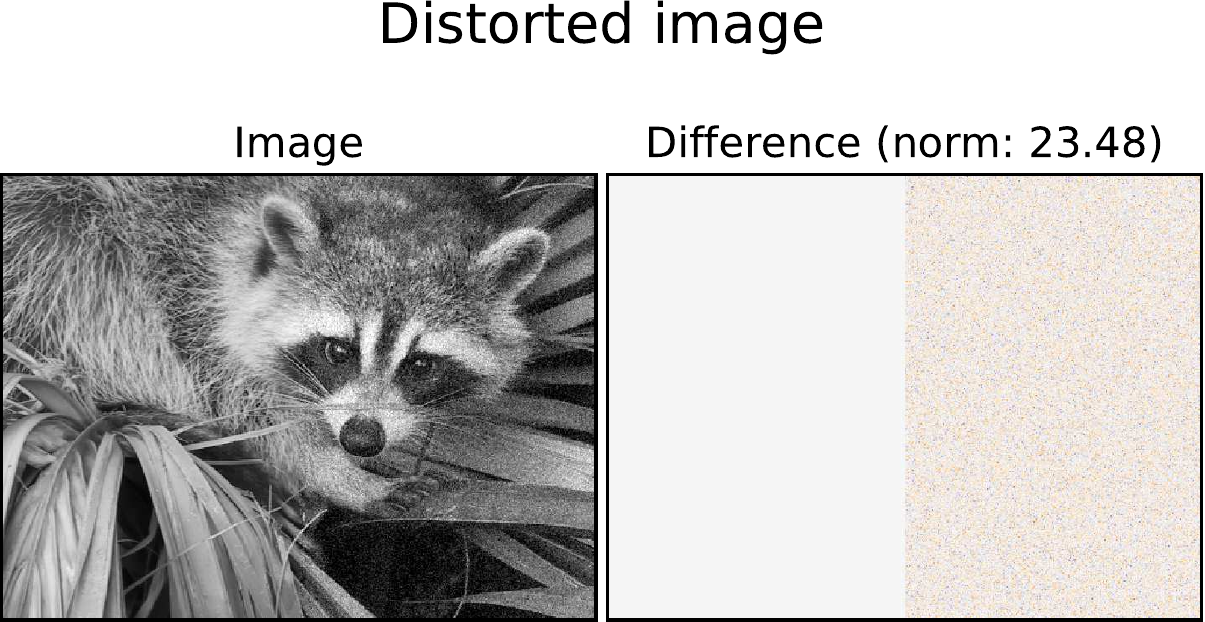}} \hspace{0pt}
\subfigure[]{\label{Fig6b}
\includegraphics[width= .48\linewidth]{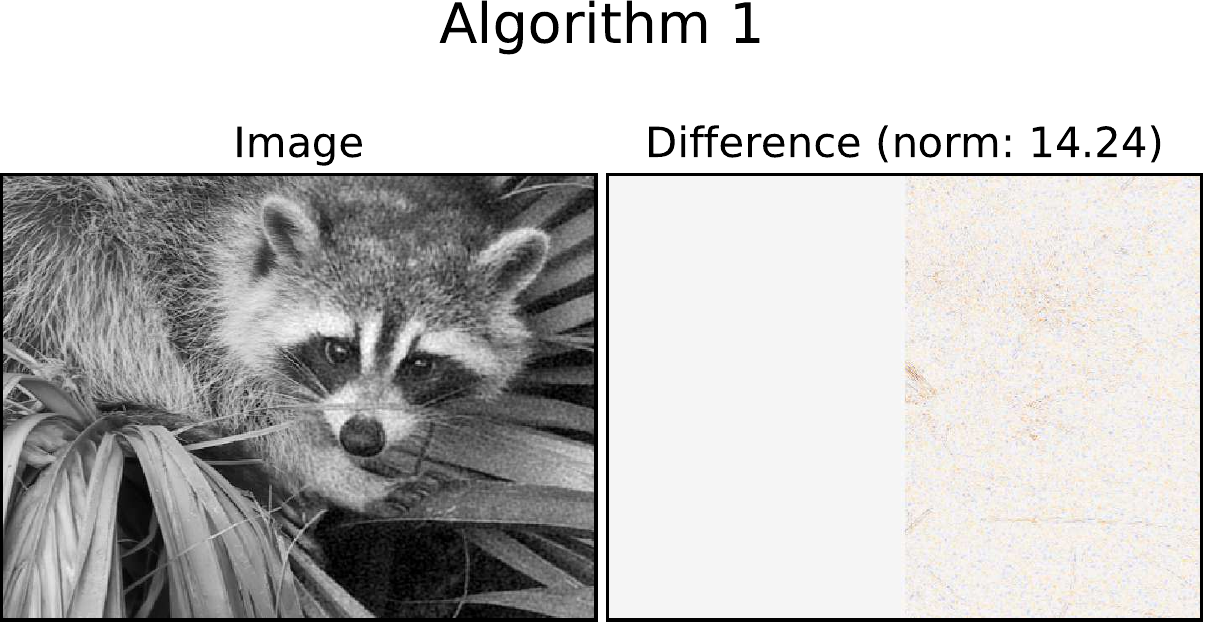}} \hspace{0pt}
\subfigure[]{\label{Fig6c}
\includegraphics[width= .48\linewidth]{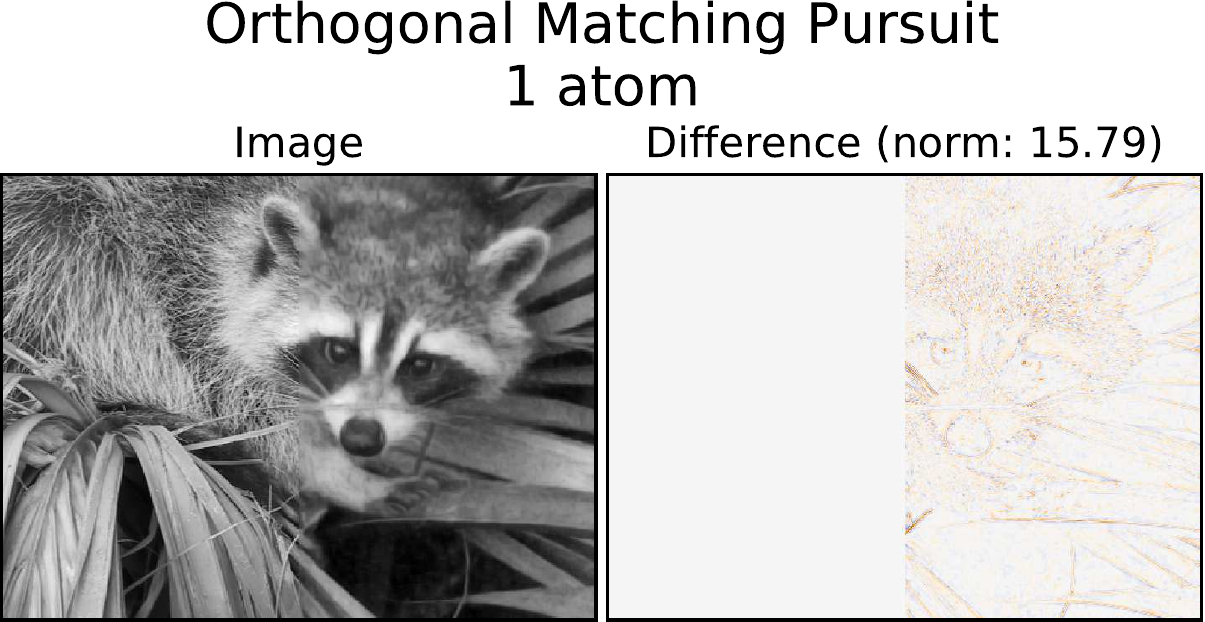}} \hspace{0pt}
\subfigure[]{\label{Fig6d}
\includegraphics[width= .48\linewidth]{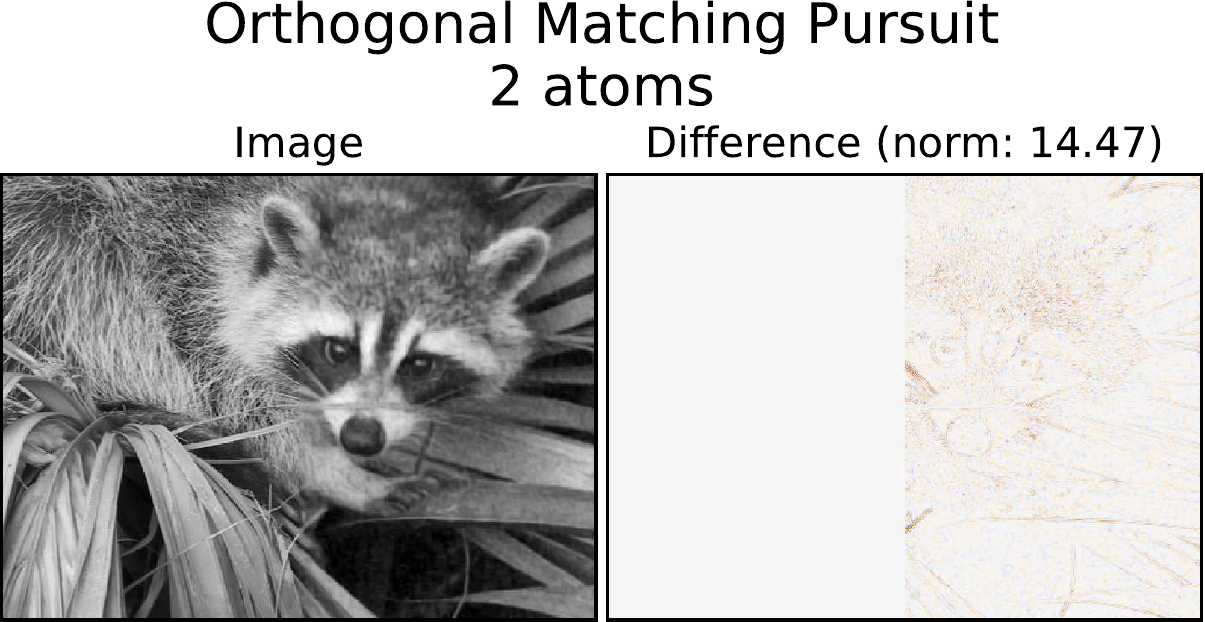}} \hspace{0pt}
\vskip -0.2 in\caption{Image denoising. (\textbf{a}): the right half of the
original image is distorted by Gaussian noise and the norm of the difference
between the distorted image and the original image is $23.48$. (\textbf{b}):
image denoising by our method (Algorithm 1), with $14$ bases used.
(\textbf{c}) and (\textbf{d}): image denoising using dictionary learning, with
$100$ bases used.}%
\label{Fig6}%
\end{figure}

\end{document}